%% file: main.tex
\documentclass[10pt,twocolumn,letterpaper]{article}
\usepackage{algorithm}
\usepackage{algpseudocode}
\usepackage{cvpr}              
\usepackage{threeparttable}
\input{preamble}
\definecolor{cvprblue}{rgb}{0.21,0.49,0.74}
\usepackage[pagebackref,breaklinks,colorlinks,allcolors=cvprblue]{hyperref}

\usepackage{titletoc}
\usepackage{appendix}

\usepackage[accsupp]{axessibility}  

\def\paperID{} 
\def\confName{CVPR}
\def\confYear{2026}
\usepackage{tocloft}

\include{our_imports}

\title{GSNR: Graph Smooth Null-Space Representation for Inverse Problems}

\author{
Romario Gualdr\'on-Hurtado\footnotemark[1]$\,$, %
Roman Jacome\footnotemark[1]$\,$, %
Rafael S. Su\'arez, 
Henry Arguello\\[.1em]
Universidad Industrial de Santander, Colombia, 680002\\[.1em]
\small{\texttt{\{yesid2238324,rajaccar,rafael2269082\}@correo.uis.edu.co}, \ \texttt{henarfu@uis.edu.co}}
}

\begin{document}
\maketitle
 \input{sec/3_finalcopy}

\clearpage


\section*{Acknowledgements}
{This work was supported by the OFFICE OF NAVAL RESEARCH under Grant W911NF2510165, Project titled “ADVANCED ALGORITHMS FOR DISTRIBUTED BEAMFORMING". The views and conclusions contained in this document are those of the authors and should not be interpreted as representing the official policies, either expressed or implied, of the OFFICE OF NAVAL RESEARCH or the U.S. Government.}

{
    \small
    \bibliographystyle{ieeenat_fullname}
    \bibliography{main}
}
\clearpage
\input{sec/X_suppl}


\end{document}

%% file: our_imports.tex
\newcommand{\boxedthm}[1]{
\begin{tcolorbox}[colback=gray!5,
                  colframe=black,
                  width=\linewidth,
                  arc=1mm, auto outer arc,
                  boxrule=1pt,
                  boxsep=-1mm,
                 ]
  #1
\end{tcolorbox}}
\usepackage{tikz,lipsum}
\usepackage[most]{tcolorbox}

\newtcolorbox{Box1}[2][]{
                lower separated=false,
                colback=white,
colframe=white!20!gray,fonttitle=\bfseries,
colbacktitle=white!10!gray,enhanced,
attach boxed title to top left={xshift=0.1cm,
        yshift=-0.01mm}, 
title=#2}
\newtheorem{theorem}{Theorem}
\makeatletter
\newcommand{\settheoremtag}[1]{
  \let\oldthetheorem\thetheorem
  \renewcommand{\thetheorem}{#1}
  \g@addto@macro\endtheorem{
    \addtocounter{theorem}{-1}
    \global\let\thetheorem\oldthetheorem}
  }
\makeatother

\newtheorem{corollary}[theorem]{Corollary}

\newtheorem{definition}{Definition}

\newtheorem{proposition}{Proposition}
\newtheorem{remark}{Remark}

\usepackage{algorithm}  
\usepackage{algpseudocode}  
\usepackage{multirow}
\usepackage{rotating} 
\usepackage{makecell}
\usepackage{array}
\usepackage[normalem]{ulem}
\usepackage{tabularray}
\usepackage{pifont}
\usepackage{amssymb}
\usepackage{pifont}

\newcommand{\Pin}{\mathbf{P}_{\!n}}

\newcommand{\Hpinv}{\mathbf{H}^\dagger}

\newcommand{\diag}{\operatorname{diag}}
\newcommand{\Cov}{\operatorname{Cov}}
\newcommand{\Var}{\operatorname{Var}}

\newcommand{\Range}{\operatorname{Range}}
\newcommand{\Null}{\operatorname{Null}}

\makeatletter
\def\smalloverbrace#1{\mathop{\vbox{\m@th\ialign{##\crcr\noalign{\kern3\p@}%
  \tiny\downbracefill\crcr\noalign{\kern3\p@\nointerlineskip}%
  $\hfil\displaystyle{#1}\hfil$\crcr}}}\limits}
\makeatother

\makeatletter
\def\smallunderbrace#1{\mathop {\vtop {\m@th \ialign {##\crcr $\hfil \displaystyle {#1}\hfil $\crcr \noalign {\kern 3\p@ \nointerlineskip } \tiny\upbracefill \crcr \noalign {\kern 3\p@ }}}}\limits}
\makeatother
\definecolor{customgreen}{HTML}{D9FFD9}
\definecolor{customblue}{HTML}{D9D9FF}
\definecolor{customteal}{HTML}{BFDFDF}
\definecolor{customorange}{HTML}{FFDFBF}

\definecolor{newgreen}{HTML}{2c9f2c}
\definecolor{newblue}{HTML}{1e76b3}

\usepackage{xcolor}
\newcommand{\best}[1]{\colorbox{customteal}{\textbf{#1}}}
\newcommand{\second}[1]{\colorbox{customblue}{\underline{#1}}}

\usepackage{dblfloatfix}

%% file: sec/3_finalcopy.tex
\begingroup
\renewcommand\thefootnote{\fnsymbol{footnote}} 
\maketitle
\footnotetext[1]{Equal contribution.}
\endgroup

\begin{abstract}
Inverse problems in imaging are ill-posed, leading to infinitely many solutions consistent with the measurements due to the non-trivial null-space of the sensing matrix. Common image priors promote solutions on the general image manifold, such as sparsity, smoothness, or score function. However, as these priors do not constrain the null‑space component, they can bias the reconstruction. Thus, we aim to incorporate meaningful null-space information in the reconstruction framework. Inspired by smooth image representation on graphs, we propose \textit{Graph-Smooth Null-Space Representation} (GSNR), a mechanism that imposes structure only into the invisible component. Particularly, given a graph Laplacian, we construct a null-restricted Laplacian that encodes similarity between neighboring pixels in the null-space signal, and we design a low-dimensional projection matrix from the $p$-smoothest spectral graph modes (lowest graph frequencies). This approach has strong theoretical and practical implications: i) improved convergence via a null-only graph regularizer, ii) better coverage, how much null‑space variance is captured by $p$ modes, and iii) high predictability, how well these modes can be inferred from the measurements. GSNR is incorporated into well-known inverse problem solvers, e.g., PnP, DIP, and diffusion solvers, in four scenarios: image deblurring, compressed sensing, demosaicing, and image super-resolution, providing consistent improvement of up to 4.3 dB over baseline formulations and up to 1 dB compared with end-to-end learned models in terms of PSNR. 
\vspace{-0.67cm}
\end{abstract}

\section{Introduction}
\vspace{-0.25cm}
 Recovering an image from a limited number of noisy measurements requires estimating a solution to an undetermined inverse problem, which is of the form 
 \vspace{-0.2cm}
\begin{equation}
 \vspace{-0.2cm}
\mathbf{y}=\mathbf{H}\mathbf{x}^\ast+\boldsymbol{\omega},\qquad \boldsymbol{\omega}\sim\mathcal{N}(\mathbf{0},\sigma^2\mathbf{I}). 
\label{eq:inverse_problem}
\end{equation}
where \(\mathbf{y}\in\mathbb{R}^m\) denotes the measurement vector, \(\mathbf{x}^\ast\in\mathbb{R}^n\) is the unknown target image, \(\mathbf{H} \in \mathbb{R}^{m\times n}\) is the sensing matrix with \(m\!\leq\!n\), and \(\boldsymbol{\omega}\) models additive Gaussian noise.
This formulation is used in a broad range of imaging tasks by using a proper structure of \(\mathbf{H}\). Compressed sensing (CS) tasks employ dense sensing matrices \cite{cs}. Image restoration problems, such as deblurring \cite{zhang2022deep}, inpainting \cite{yu2018generative}, and super-resolution  \cite{tian2011survey}, are modeled by structured Toeplitz or block-diagonal matrices. In medical imaging, modalities like magnetic resonance imaging (MRI) \cite{lustig2008compressed} and computed tomography (CT) \cite{baguer2020computed} use structured physics-based operators (e.g., undersampled Fourier or Radon transforms).
The formulation in \eqref{eq:inverse_problem} becomes ill-posed due to matrix instability and/or the undersampling scheme. Therefore, incorporating prior knowledge or regularization is essential to promote reconstructions that are consistent with the expected structure of \(\mathbf{x}^\ast\) \cite{benning2018modern}.
The estimation of \(\mathbf{x}^\ast\) is commonly formulated as a variational problem that combines a data-fidelity term and a regularization term:
 \vspace{-0.19cm}
\begin{equation}
 \vspace{-0.19cm}
\widehat{\mathbf{x}}=\arg\min_{\mathbf{\tilde{x}}}\ \underbrace{\tfrac12\|\mathbf{H}\mathbf{\tilde{x}}-\mathbf{y}\|_2^2}_{\text{data fidelity:}\ g(\mathbf{\tilde{x}})}+\eta\, \underbrace{f(\mathbf{\tilde{x}})}_{\text{prior}} \label{eq:variational}
\vspace{-0.1cm}
\end{equation}
where \(g(\mathbf{\tilde{x}})\) enforces consistency with the measurements and \(f(\mathbf{\tilde{x}})\) incorporates prior knowledge about the structure of \(\mathbf{x}^\ast\).
Recent advances have introduced data-driven regularizers for \(f(\mathbf{\tilde{x}})\), most notably Plug-and-Play (PnP) \cite{kamilov2023plug, venkatakrishnan2013plug,chan2016plug,zhang2021plug} and Regularization by Denoising (RED) \cite{romano2017little,cohen2021regularization,reehorst2018regularization}, which have achieved remarkable success across a wide range of imaging tasks \cite{kamilov2023plug} by implicitly promoting solutions within the training image dataset manifold i.e., well-trained denoiser learns the data score-function \cite{milanfar2025denoising}. PnP and RED use learned denoisers, allowing the integration of powerful deep restoration networks within algorithms \cite{tan2025image,milanfar2025denoising}. Despite their versatility and success, these priors operate in the image domain and do not explicitly exploit the structure of \(\mathbf{H}\), including its null-space (NS). Only a portion of \(\mathbf{x}^\ast\) is observable through $\mathbf{H}$. The range-null-space decomposition (RNSD) states that any $\mathbf{x} \in \mathbb{R}^n$ vector is decomposed as $\mathbf{x} = \mathbf{x}_r + \mathbf{x}_n$, where $\mathbf{x}_r=\mathbf{P}_r \mathbf{x}$ denotes the component lying in the range space of $\mathbf{H}$, with $\mathbf{P}_r=\Hpinv\mathbf{H}$, and $\mathbf{x}_n=\Pin \mathbf{x}$ denotes the component lying in its NS, with the NS projector \(\Pin=\mathbf{I}-\Hpinv\mathbf{H}\). $\mathbf{x}_n$ is invisible to $\mathbf{H}$ because it belongs its NS, defined as
 \vspace{-0.2cm}
\begin{align*}
 \vspace{-0.2cm}
\Null(\mathbf{H}) 
    &= \{\mathbf{x} \in \mathbb{R}^n : \mathbf{H}\mathbf{x} = \mathbf{0}\} \\
    &= \{\mathbf{x} : \mathbf{x} \perp \mathbf{h}_j,\; \forall j \in \{1, \ldots, m\}\},
 \vspace{-0.2cm}
\end{align*}
where $\mathbf{h}_j\in\mathbb{R}^n$ is the $j$-th row of $\mathbf{H}$.
This interpretation has inspired approaches that explicitly exploit  \(\Null(\mathbf{H})\). Null-Space Networks (NSN) \cite{schwab2019deep} and Deep Decomposition \cite{chen2020deep}  operate across the entire \(\Null(\mathbf{H})\), treating all invisible directions as equally relevant, neglecting the fact that natural, perceptually plausible images occupy only a low-dimensional manifold within that space \cite{Neurips}. To address this gap, the recent Nonlinear Projections of the Null-Space (NPN) framework \cite{Neurips} introduces a learnable low-dimensional NS projection matrix \(\mathbf{S}\in\mathbb{R}^{p\times n}\), with $p\leq(n-m)$. $\mathbf{S}$ is designed depending on the inverse problem, making its rows orthogonal to $\mathbf{H}$, then $\mathbf{S}$ is optimized jointly a NS component predictor \(\mathrm{G}(\mathbf{y})\equiv \mathbf{S}\mathbf{x}^\ast\) to adapt the $\mathrm{span}(\mathbf{S})$ into a subspace of $\mathrm{Null}(\mathbf{H})$. During reconstruction, NPN penalizes deviations \(\|\mathrm{G}(\mathbf{y})-\mathbf{S}\mathbf{\tilde{x}}\|^2_2\). However,
natural images occupy a low‑dimensional, structured subset inside $\Null(\mathbf{H})$, so
blindly learning an arbitrary $\mathrm{span}(\mathbf{S}^\top)\subset\Null(\mathbf{H})$ can waste capacity and induce bias. In this work, we present a principled framework that provides a new representation of the NS, enabling the optimal selection of informative NS directions. To address this issue, we took into account two main optimality criteria: 

\boxedthm{
    \textbf{Coverage:}  \textit{How much of the null-space of $\mathbf{H}$ is represented with the projection matrix $\mathbf{S}$.}

\textbf{Predictability:} \textit{How easily can the projections $\mathbf{Sx}$ be estimated from the measurements $\mathbf{y}$ only. }
}

Inspired by the graph (smooth) representation of images \cite{ortega2018graph,6319643}, we introduce Graph-Smooth Null-Space Representation (GSNR)
to overcome the shortcomings of NS representation and prediction.  Different from the well-studied graph-structure preserving approach in imaging inverse problems \cite{kalofolias2016learn,6319643} that promotes smoothness over the image $\mathbf{x}$, our approach is the first to endow the NS component $\mathbf{x}_n$ with a graph structure, defined through a sparse, positive semidefinite graph Laplacian $\mathbf{L}\succeq\mathbf{0}$, such as 4/8‑nearest‑neighbor (NN) image grid whose edge weights encode local pixel similarity. We form the \emph{null‑restricted Laplacian} $\mathbf{T}\ :=\ \mathbf{P}_n\,\mathbf{L}\,\mathbf{P}_n\quad\text{on}\quad\Null(\mathbf{H})$. Via spectral graph theory, we construct the graph-smooth NS projection matrix $\mathbf{S}$ as the $p$-smoothest NS modes, which corresponds to the $p$ first graph Fourier modes.  
The proposed GSNR changes how we solve inverse problems by providing exact structure information where the sensor is blind. By selecting the $p$-smoothest eigenmodes of the null‑restricted Laplacian, our method (i) reduces hallucinations and bias by constraining only the invisible component, not the whole image; (ii) improves conditioning, leading to faster and more stable convergence across PnP, DIP, and diffusion solvers; (iii) increases {data efficiency} by capturing most null‑space variation with small $p$ (high coverage) and by targeting the modes that are easiest to infer from measurements (high predictability); and (iv) offers operational diagnostics coverage/predictability curves that allows selecting $p$ and justify regularization objectively. Because it is plug‑compatible with standard pipelines and agnostic to the forward operator, the approach scales across different imaging tasks, turning an ill‑posed ambiguity into a structured, measurable, and learnable component. The proposed construction has the following benefits: 

\begin{enumerate}
    \item Selected graph-smooth null-space components produce high coverage at low dimensionality projection, i.e., small $p$ (Theorem \ref{th:coverage} and Theorem \ref{th:Th1}).

    \item 
Predicting the $p$ null-space components from $\mathbf{y}$ using a GSNR is easier than the ones obtained with a common null-space basis (Proposition \ref{thm:modeR2-bold}). 

\item Predicted graph null-space components can be incorporated via regularization into solvers such as PnP, Deep Image Prior, and Diffusion Models.
\end{enumerate}

\vspace{-0.3cm}
\section{Related Work}
\vspace{-0.2cm}

\hspace{0.25cm} \textbf{Variational and learned priors.} A common scheme to solve \eqref{eq:variational} is proximal gradient descent (PGD) \cite{parikh2014proximal}. Recently, PnP \cite{kamilov2023plug} and RED \cite{romano2017little} methods have replaced the proximal operator of \(f(\mathbf{\tilde{x}})\) with learned denoisers, enabling the incorporation of data-driven priors without requiring an explicit analytic form \cite{Neurips}.
Although these frameworks achieve state-of-the-art performance, they leave $\mathrm{Null}(\mathbf{H})$ largely unconstrained; denoisers may freely modify components that are invisible to the sensing operator.

\textbf{Nonlinear Projections of the Null-Space (NPN).} To explicitly regularize image unobserved components, \cite{Neurips} introduced a task-specific low-dimensional projection within $\mathrm{Null}(\mathbf{H})$, enforcing consistency between learned NS predictions and reconstructions.
Building on this principle, we propose GSNR, which incorporates graph-smoothness analysis to guide the selection of a structured null-subspace.
Our design ensures that the learned projection aligns with the intrinsic geometry of the data, yielding NS components that are both predictable from the measurements and consistent with the graph-induced smoothness prior.

\textbf{Graph theory in imaging inverse problems}. Variational regularization methods address inverse problems by imposing smoothness or sparsity priors that constrain the solution space. Total variation \cite{RUDIN1992259} enforces pixel-wise smoothness, whereas wavelet-based sparsity priors \cite{10.1093/biomet/81.3.425} promote sparse representations of images in the transform domain via soft-thresholding.
Although these priors improve stability, they are typically limited to local structures and often oversmooth fine textures or complex geometries.
To better capture local dependencies, graph-based and non-local regularization methods have been reintroduced. \cite{peyre2008non} proposed an adaptive nonlocal regularization framework that jointly estimates the image and the graph that encodes patch similarities.  However, like classical priors, this framework operates solely on the reconstruction and doesn't constrain the image component lying in $\mathrm{Null}(\mathbf{H})$.

\vspace{-0.2cm}
\section{Graph-Smooth Null-Space Representation}

\vspace{-0.1cm}
\subsection{Preliminaries: Graphs}
\vspace{-0.2cm}
Graphs provide a geometric representation of data, where vertices correspond to data elements (e.g., pixels or image patches) and weighted edges encode pairwise similarities between them \cite{shuman2013emerging}.
Signals defined on weighted undirected graphs are mathematically denoted by $\mathcal{G}=(\mathcal{V},\mathcal{E},\mathbf{W})$, where $\mathcal{V}$ is the vertex set with $|\mathcal{V}|=n$, $\mathcal{E}\subseteq\mathcal{V}\times\mathcal{V}$ is the edge set, and $\mathbf{W}\in\mathbb{R}^{n\times n}$ is the symmetric weighted adjacency matrix with $\mathbf{W}_{ij}>0$ if $(i,j)\in\mathcal{E}$ and $\mathbf{W}_{ij}=0$ otherwise. The unnormalized graph Laplacian is
\vspace{-0.2cm}
\begin{equation}
\vspace{-0.2cm}
\mathbf{L}=\mathbf{D}-\mathbf{W},\qquad
(\mathbf{L}\mathbf{x})_i=\sum_{j}\mathbf{W}_{ij}\bigl(\mathbf{x}_i-\mathbf{x}_j\bigr),
\vspace{-0.1cm}
\end{equation}
where the diagonal degree matrix $\mathbf{D}=\mathrm{diag}(d_1,\dots,d_n)$ accomplishes $d_i=\sum_{j}\mathbf{W}_{ij}$. Its Dirichlet energy is
\vspace{-0.2cm}
\begin{equation}
\vspace{-0.2cm}
\mathbf{x}^{\top}\mathbf{L}\mathbf{x}=\frac{1}{2}\sum_{i,j}\mathbf{W}_{ij}\bigl(\mathbf{x}_i-\mathbf{x}_j\bigr)^2, \label{eq:Dirichlet}
\vspace{-0.1cm}
\end{equation}
which penalizes discrepancies across edges with large weights. For symmetric $\mathbf{W}$, $\mathbf{L}$ is real symmetric positive semidefinite (SPSD) with eigenpairs $\{(\lambda_\ell,\mathbf{v}_\ell)\}_{\ell=0}^{n-1}$, $0=\lambda_0\le \lambda_1\le\cdots$, $\mathbf{v}_0\propto\mathbf{1}$ on connected graphs. Eigenvectors with small $\lambda_\ell$ vary slowly across high-weight edges, motivating the interpretation of $\lambda$ as a graph frequency. For the usage of  $\mathbf{L}$, we focus on two graph topologies \cite{ortega2018graph}:

\begin{enumerate}
    \item \textbf{4NN:} This promotes piecewise smoothness aligned with the coordinate axes and attenuates horizontal/vertical fluctuations, yielding an anisotropic (axis-biased) smoothing effect. Here, $\mathcal{E}_4=\{(i,j):\|\mathbf{p}_i-\mathbf{p}_j\|_1=1\}$, $\mathbf{W}_{ij}=\mathbf{1}_{(i,j)\in\mathcal{E}_4}$, with $\mathbf{p}_i\in\mathbb{Z}^2$ pixel coordinates.

    \item \textbf{8NN:} This reduces orientation bias and more closely approximates an isotropic Laplace operator, promoting rotation-invariant piecewise smoothness and suppressing fluctuations uniformly across directions.  Here, $\mathcal{E}_8=\{(i,j):\|\mathbf{p}_i-\mathbf{p}_j\|_\infty=1\}$, $\mathbf{W}_{ij}=w_o\mathbf{1}_{\|\cdot\|_2=1}+w_d\mathbf{1}_{\|\cdot\|_2=\sqrt{2}}$, where $w_d=w_o/\sqrt{2}$.


\end{enumerate}

\noindent Nevertheless, we extend the analysis of the proposed approach to other topologies in Supp. \ref{app:additional_graphs}. 

\vspace{-0.2cm}
\subsection{Graph-smooth null modes}\label{subsect:graphlimited}
\vspace{-0.2cm}

We study smoothness within the NS by restricting the Laplacian to
\vspace{-0.1cm}
\begin{equation}
\vspace{-0.1cm}
\mathbf{T}=\mathbf{P}_n \mathbf{L} \mathbf{P}_n, \label{eq:T_definition}
\end{equation}

\noindent 
the reason for using the null-graph operator $\mathbf{T} \in \mathbb{R}^{n \times n}$ stems from the need to pay attention to certain geometric properties of the image's NS, which are extracted by $\mathbf{L}$. Let $\mathbf{Tx}=\mathbf{P}_n\mathbf{L}\mathbf{P}_n\mathbf{x}=\mathbf{P}_n\mathbf{L}\mathbf{x}_n$, where $\mathbf{L}\mathbf{x}_n$ extracts, areas of greatest difference in the NS, and then re-project it to the NS with $\mathbf{P}_n$, thus ensuring that this operation results in a projection from the NS.
In Fig. \ref{fig:proj_ns}, from left to right, we show the ground-truth $\mathbf{x}^*$, back-projection $\mathbf{H}^\top \mathbf{y}$, true null component $\mathbf{P}_n \mathbf{x}^*$, and two graph NS projections $\mathbf{P}_n \mathbf{L}_{8\text{nn}} \mathbf{P}_n \mathbf{x}^*$ and $\mathbf{P}_n \mathbf{L}_{4\text{nn}}\mathbf{P}_n \mathbf{x}^*$, for super-resolution (SR).
The map $\mathbf{P}_n \mathbf{x}^*$ isolates signal content invisible to the sensor, while
$\mathbf{P}_n \mathbf{L}\mathbf{P}_n  \mathbf{x}^*$ highlights {where} $\mathbf{Lx}_n$ falls into the NS. These graph projections highlight the smoothest \textit{blind} signal components, e.g., textures near edges.

\begin{figure}[!t]
    \centering
    \includegraphics[width=\linewidth]{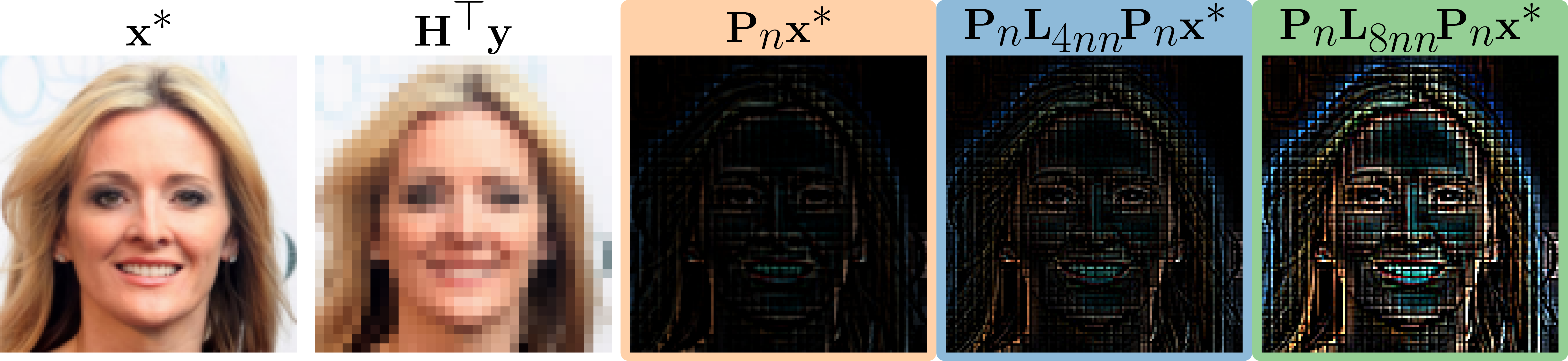}
\vspace{-0.7cm}
    \caption{For image SR task with $SRF=4, n= 3 \cdot 128^2$, we show ground-truth, adjoint reconstruction, NS projection, projection onto graph-smooth NS with $\mathbf{L}_{4nn}$ and $\mathbf{L}_{8nn}$.}
\vspace{-0.5cm}
    \label{fig:proj_ns}
\end{figure}

Then, by applying eigenvalue decomposition, we see that $\mathbf{T} =\mathbf{V} \operatorname{diag}\left(\mu_1, \ldots, \mu_n\right) \mathbf{V}^{\top}$ is SPSD, with $\quad 0 \leq \mu_1 \leq \cdots \leq \mu_n $, where $\mathbf{V} \in \mathbb{R}^{n \times n}$ is orthogonal. Note that $\mathbf{V}$ diagonalizes $\mathbf{T}$, so its columns $\{\mathbf{v}_i\}$ are the \emph{graph-Fourier modes within $\mathrm{Null}(\mathbf{H})$}. 
If $\mu_i$ is smaller, then a smoother graph-mode is selected. $\mathbf{V}$ has the following properties:  i) $\mathbf{V}$ has orthonormal columns $\mathbf{v}_i\in\mathrm{Null}(\mathbf{H})$, $\mu_i$-ordered. ii) These $\{\mathbf{v}_i\}$ are \emph{graph-smooth null modes} probing $\mathbf{x}$. The graph-smooth NS projector $\mathbf{S}$ with $p$ rows is obtained by taking the first $p$ columns of $\mathbf{V}$:
\vspace{-0.2cm}
\begin{equation}
\vspace{-0.2cm}
\mathbf{S}=\overbrace{\mathbf{V}[:,\,1\!:\!p]^{\top}}^{\mathbf{\mathbf{V}_p^\top}}  = [\mathbf{v}_1,\dots, \mathbf{v}_p]^\top  \in \mathbb{R}^{p \times n}. \label{eq:S_definition}
\end{equation} 

  \noindent Note that $\mathbf{S}$ keeps only the $p$-smoothest null directions and $\mathbf{Sx}$ expresses the NS part of $\mathbf{x}$ in an orthonormal basis sorted by increasing graph frequency.  Particularly, based on RNSD: $\mathbf{S}\mathbf{x}
  =  \mathbf{S}\bigl(\mathbf{P}_r+\mathbf{P}_n\bigr)\mathbf{x} =  \mathbf{S}\mathbf{P}_r\mathbf{x}+\mathbf{S}\mathbf{P}_n\mathbf{x},$ since rows of $\mathbf{S}$ lie in $\mathrm{Null}(\mathbf{H})$, hence $\mathbf{S}\mathbf{P}_r=\mathbf{0}$, leading to $
  \mathbf{Sx}= \mathbf{0} + \mathbf{S}\mathbf{P}_n\mathbf{x} = \mathbf{V}_p^\top\mathbf{P}_n\,\mathbf{x} = \begin{bmatrix}
      \langle \mathbf{v}_1,\mathbf{x}_n\rangle\\[-2pt]\vdots\\[-2pt]
      \langle \mathbf{v}_p,\mathbf{x}_n\rangle
     \end{bmatrix}.$
\noindent Here $\{\mathbf{v}_i\}_{i=1}^p$ are the $p$ Laplacian-smoothest orthonormal null modes; thus $\mathbf{S}\mathbf{x}$
collects the coefficients of $\mathbf{x}$ along these $p$ directions.

\vspace{-0.2cm}
\subsection{Learning low-dimensional GSNR}

\vspace{-0.2cm}
We train the GSNR predictor \(\mathrm{G}\) following:
\vspace{-0.2cm}
\begin{equation}
\vspace{-0.2cm}
\mathrm{G}^{\ast} = \min_{\mathrm{G}}~~\mathbb{E}\big[\|\mathrm{G}(\mathbf{y})-\mathbf{S} \mathbf{x}^\ast\|_2^2\big].\label{eq:train-g}
\end{equation}

\noindent Here, we aim to predict the low-dimensional GSNR using the neural network $\mathbf{G}$. Subsequently, the theoretical foundations will be laid that greater predictability will guarantee better convergence in the task.

\vspace{-0.2cm}
\subsection{Reconstruction objective}

\vspace{-0.2cm}
Thus, we incorporate our method as
\vspace{-0.2cm}
\begin{equation}
\vspace{-0.2cm}
\min_{\mathbf{\tilde{x}}} g(\mathbf{\tilde{x}})+\eta\,f(\mathbf{\tilde{x}})+\gamma\|\mathrm{G}^{\ast}(\mathbf{y})-\mathbf{S} \mathbf{\tilde{x}}\|_2^2+ \tfrac{\gamma_{g}}{2}\overbrace{\mathbf{\tilde{x}}^{\top} \mathbf{T}\mathbf{\tilde{x}}}^{\phi(\mathbf{\tilde{x})}},\label{eq:recon}
\end{equation}
which includes the learned GSNR term, plus a graph regularizer acting {only} on the null component of the reconstruction, weighted by $\gamma$ and $\gamma_g$, respectively. This retains compatibility with proximal/PnP/DM solvers. 

\vspace{-0.2cm}
\section{Theory and Analysis}
\vspace{-0.2cm}

We develop a theoretical analysis on GSNR based on \emph{null-restricted spectral} formulation that turns GSNR design/performance prediction into ordered-eigenvalue criteria on $\mathbf{T}$ for optimal subspace selection,  obtaining a principled eigen-ordering and criteria used to choose $\mathbf{S}, p, \mathbf{L}$. The theoretical results are summarized as:

\boxedthm{
\uline{\textit{Coverage}}: the first \(p\) eigenmodes of \(\mathbf{T}\) provide an optimal graph-smooth representation of \(\operatorname{Null}(\mathbf{H})\), by covering a high null-space spectrum energy with~$p\ll n-m$.

\uline{\textit{Minimax optimality}}: the same GSNR basis is worst-case optimal over a graph-energy ellipsoid in \(\operatorname{Null}(\mathbf{H})\), giving a principled ordering of null directions

\uline{\textit{Predictability}}: via block-precision/Schur-complement analysis, we show that GSNR has a higher per-mode recoverability from measurements $\mathbf{y}$ than plain null-space projections.

\uline{\textit{Convergence}:} the effect of the graph regularizer provides better conditioning and convergence properties based on fixed-point analysis.}

\vspace{-0.2cm}
\subsection{Null-space coverage}
\vspace{-0.2cm}

Before developing the formal analysis on the coverage properties of the NS projections, we validate the representation performance of three settings: i) $\mathbf{L=I}$, ii) $\mathbf{L}_{4nn}$, and iii) $\mathbf{L}_{8nn}$. Note that the first setting is equivalent to just taking the NS basis. In Fig. \ref{fig:graph_coverage} (a), each row shows projections ($\mathbf{x}=(\mathbf{H}^\top\mathbf{H} +\mathbf{S}^\top \mathbf{S})\mathbf{x}^*$) as the dimension of the  null–space subspace increases,
\(p\in\{153,\,1536,\,4608,\,9216,\,15360\}\), for three graph choices in the design of \(\mathbf{S}\):
(top, orange) \(\mathbf{L=I}\); (middle, green) grid \(\mathbf{L}_{8\text{nn}}\); (bottom, blue) grid \(\mathbf{L}_{4\text{nn}}\).
For small \(p\), the grid Laplacians already yield visually plausible faces, while \(\mathbf{L=I}\) still produces noisy, poorly structured estimates. The number printed above each image is the MSE of the representation. The plot at the bottom shows the MSE for all values of $p = \{1,\dots, n-m\}$, showing that the graph-null-space projections provide better signal representation with smaller $p$ than just the NS basis. In general, to measure the amount of energy from the NS that the low-dimensional operator $\mathbf{S}$ can capture, we use the spectral coverage defined as:
\vspace{-0.2cm}
\begin{equation}
\vspace{-0.2cm}    C(p)=\frac{\operatorname{tr}\left(\mathbf{S}\mathrm{Cov}(\mathbf{x}_n)\mathbf{S}^{\top}\right)}{\operatorname{tr}\left(\mathrm{Cov}(\mathbf{x}_n)\right)} \in[0,1],\label{eq:Cp}
\end{equation}
where  $\Cov(\mathbf{x})$ is the covariance matrix of $\mathbf{x}$. View the Laplacian eigenvectors as a {graph–Fourier} basis: the first modes are smooth, low “graph–frequency” patterns, while later modes are increasingly oscillatory \citep{shuman2013emerging,ortega2018graph}. The $C(p)$ indicates what fraction of the total {NS variance} lies in the span of the first \(p\) graph–Fourier modes retained by $\mathbf{S}$, similarly as explained variance in PCA \cite{jolliffe2016principal}. Intuitively, if the signal is graph–smooth, most variability concentrates in low frequencies, so coverage climbs rapidly with small \(p\); if not, coverage grows more slowly, like in PCA. \ref{app:graph_as_a_matrix} shows coverage results and eigenvalue analysis that validate the usefulness of {graph Laplacians}, rather than a learnable matrix. Mathematical insights for the selection of the appropriate $p$ value depending on the coverage are given in \ref{app:t_constructio}.

\begin{figure}[!t]
    \centering
    \includegraphics[width=\linewidth]{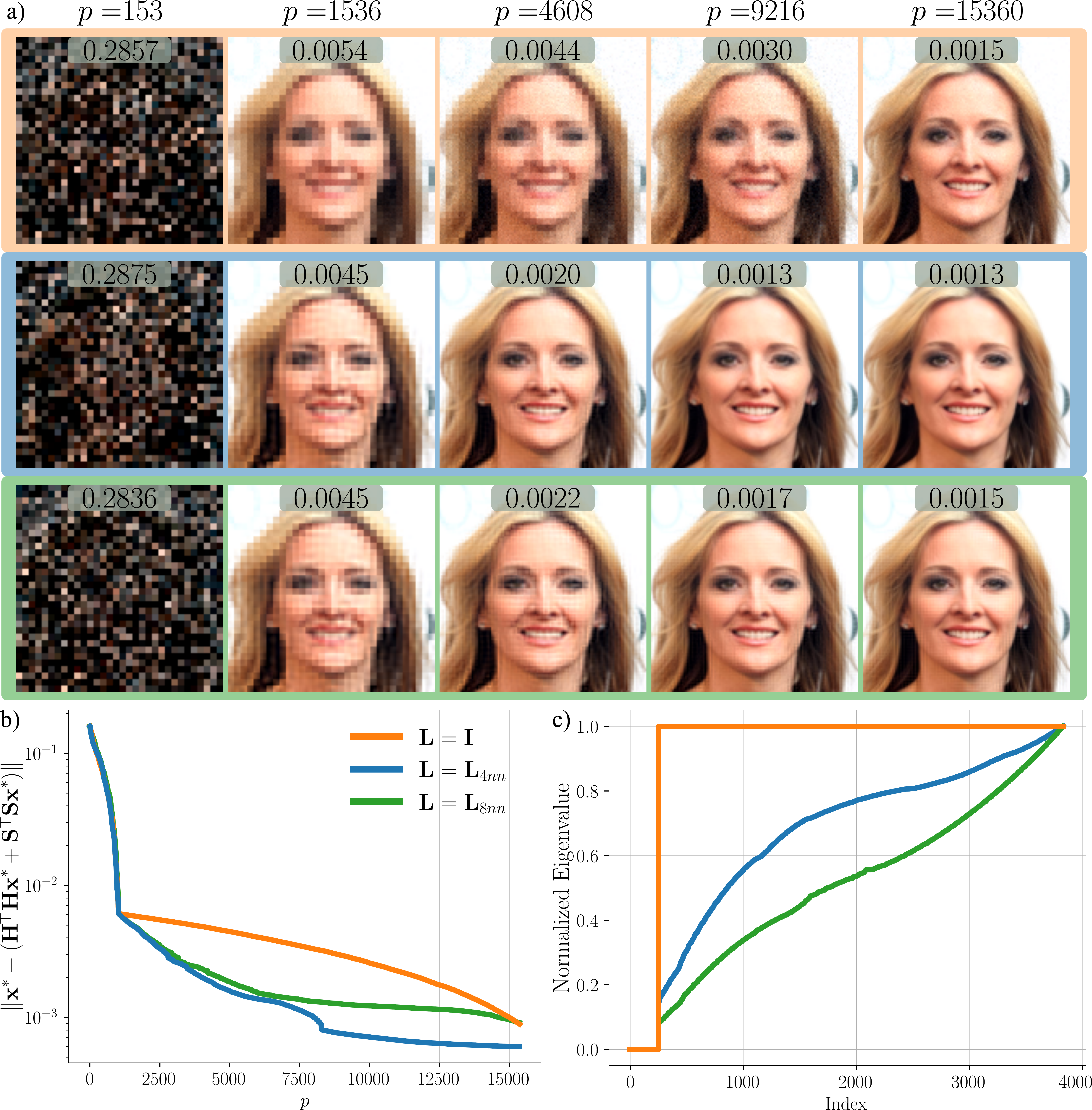}
\vspace{-0.7cm}
    \caption{Coverage and spectral analysis for SR task. \textbf{(a)} $\mathbf{x^\ast} = \mathbf{H}^\top \mathbf{Hx^\ast+S}^\top\mathbf{Sx^\ast}$ for $\mathbf{S}$ given by \eqref{eq:S_definition} using $\mathbf{L}=\mathbf{I}$ (orange), $\mathbf{L = L}_{4nn}$ (green) and $\mathbf{L=L}_{8nn}$ (blue) for different values of $p$. \textbf{(b)} RNSD representation error, varying $p$. \textbf{(c)} Variation of $\mathbf{T}$ normalized eigenvalues with respect to their index.}
    \label{fig:graph_coverage}
\vspace{-0.5cm}
\end{figure}

\vspace{-0.2cm}
\begin{definition}[Gaussian Markov Random Field \cite{rue2005gaussian}]

A GMRF on the graph $\mathcal{G}= (\mathcal{V},\mathcal{E})$ is a zero-mean Gaussian vector $\mathbf{x}\in \mathbb{R}^n$ with precision matrix $\mathbf{Q} = \boldsymbol{\Sigma}^{-1} \in \mathbb{R}^{n\times n}$, with $\boldsymbol{\Sigma} = \operatorname{Cov}(\mathbf{x})$, is sparse and encodes the conditional independence of the graph 
\vspace{-0.2cm}
\begin{equation*}
\vspace{-0.2cm}
    p(\mathbf{x}) \propto \operatorname{exp} \left(-\frac{1}{2} \mathbf{x}^\top \mathbf{Q}\mathbf{x}\right), \quad \mathbf{Q}_{i,j} = 0 \longleftrightarrow (i,j) \notin \mathcal{E},
\end{equation*}

\noindent where smoothness is required in neighboring pixels, we define the precision matrix $\mathbf{Q} = \alpha \mathbf{L} + \epsilon \mathbf{I}$ with $\alpha>0, \epsilon>0$.

\end{definition}

\vspace{-0.3cm}

\noindent Under this setting, we derive the following guarantee on improved NS coverage. 
\vspace{-0.2cm}
\begin{theorem}[Coverage for graph-smooth null-space]\label{th:coverage}
Consider the construction of $\mathbf{T}$ in \eqref{eq:T_definition}. The covariance of the NS, $\mathrm{Cov}(\mathbf{x}_n)$  is a spectral function of $\mathbf{T}$, i.e.,  $\mathrm{Cov}(\mathbf{x}_n) = \mathbf{V}\mathrm{diag}({\lambda_1,\dots, \lambda_n})\mathbf{V}^\top$, where $\lambda_i = \frac{1}{\alpha \mu_i + \epsilon}$, $\lambda_1\geq\lambda_2\cdots\geq \lambda_n$. With the construction of $\mathbf{S}$ in \eqref{eq:S_definition}, the coverage of the NS using GMRF with a Laplacian matrix $\mathbf{L}$ as $C_L(p)$ and the coverage when $\mathbf{L=I}$ is denoted $C_I(p)$, satisfies for every $p=1,\dots, n-m$.
\vspace{-0.2cm}
\begin{equation*}
\vspace{-0.2cm}
    C_L(p)\ \ge\ C_I(p).
\end{equation*}

\end{theorem}
The proof can be found in Supp. \ref{app:proof_cov}. 

We illustrate this effect in Fig. \ref{fig:graph_coverage}(b)-(c). In \ref{fig:graph_coverage}(b), the representation error of the RNSD decreases more rapidly when using $\mathbf{L}_{8nn}$ and $  \mathbf{L}_{4nn}$ compared to $\mathbf{I}$. In \ref{fig:graph_coverage}(c), the eigenvalue spectra of $\mathbf{T}$ reveal that, unlike the flat spectrum of the pure NS projector $\mathbf{P}_n$, graph Laplacians exhibit smoothly increasing eigenvalues. From Theorem \ref{th:coverage}, a flat spectrum ($\mu_i=1$ for all $i$ as with $\mathbf{L}=\mathbf{I}$) yields $C_{\mathbf{I}}(p)=p/q$, so coverage grows only linearly with $p$, whereas Laplacians with increasing $\{\mu_i\}$ achieve faster coverage growth.

Our second theoretical result is a minimax optimality on the NS coverage by the NS scheme. First, define the feasible set of NS components
\vspace{-0.2cm}
\begin{equation}
\vspace{-0.2cm}
\mathcal{M}_\tau
:= \bigl\{\mathbf{x}_n\in \mathrm{Null}(\mathbf{H}):\ \mathbf{x}_n^\top \mathbf{L}\,\mathbf{x}_n \le \tau\bigr\} \nonumber
\label{eq:Mtau}
\end{equation}

\begin{theorem}[Minimax optimality]\label{th:Th1}
Let \(V_p=\mathrm{span}\{\mathbf{v}_1,\dots,\mathbf{v}_p\}\) and \(P_{V_p}\) be its projector. Then, among all \(p\)-dimensional \(V\subset\mathrm{Null}(\mathbf{H})\), 
provided \(\mu_{p+1}>0\),
\vspace{-0.1cm}
\begin{align*}
\vspace{-0.1cm}
\min_{\dim(V)=p}\;
    &\sup_{\mathbf{x}_n \in \mathcal{M}_\tau}\;
    \|(\mathbf{I} - P_V)\mathbf{x}_n\|_2^2 = \\
    &\sup_{\mathbf{x}_n \in \mathcal{M}_\tau}\;
    \|(\mathbf{I} - P_{V_p})\mathbf{x}_n\|_2^2 = \frac{\tau}{\mu_{p+1}}.
\end{align*}
\end{theorem}

\noindent The proof of this theorem is in Supp. \ref{app:minmax_thm}. 
\vspace{-0.1cm}
\begin{remark}
    Selecting $\mathbf{S}$ so that its rows are $\{\mathbf{v}_1^\top,\dots,\mathbf{v}_p^\top\}$ is {provably optimal}
for covering the NS under the graph energy $\mathbf{x}^\top \mathbf{T} \mathbf{x}$.
As $p$ grows, the guaranteed worst-case error decays with $1/\mu_{p+1}$, which is fast for graph Laplacians
(where many $\mu_j$ are small). With $\mathbf{L=I}$ one gets $\mathbf{T=P}_n$ (flat spectrum), hence
$\mu_{p+1}=1$ and no improvement in the bound, explaining why geometry-free designs cannot
cover the NS efficiently.
\end{remark}

\subsection{Null-space projections predictability}
\vspace{-0.2cm}

It is also important to quantify the predictability level of the NS projections $\mathbf{Sx}^{\ast}$ from the acquired measurements $\mathbf{y}$. We develop a \textit{statistical coupling} analysis via linear estimation.

\vspace{-0.1cm}
\begin{proposition}[Per–mode predictability bound]
\label{thm:modeR2-bold}
Let $\mathbf{x}\!\in\!\mathbb{R}^n$ be zero–mean Gaussian with covariance $\mathbf{C}$ and precision
$\mathbf{Q}=\mathbf{C}^{-1}$. Let $\mathbf{H}\!\in\!\mathbb{R}^{m\times n}$ and denote the orthogonal
projectors onto $\mathcal{R}:=\Range(\mathbf{H}^\top)$ and $\mathcal{N}:=\Null(\mathbf{H})$ by
$\mathbf{P}_r$ and
$\mathbf{P}_n$, respectively.
Block $\mathbf{Q}$ and $\mathbf{C}$ with respect to the decomposition $\mathbb{R}^n=\mathcal{R}\oplus\mathcal{N}$:
\vspace{-0.2cm}
\begin{equation*}
\vspace{-0.2cm}
    \mathbf{Q}=
\begin{bmatrix}\mathbf{Q}_{rr}&\mathbf{Q}_{rn}\\ \mathbf{Q}_{nr}&\mathbf{Q}_{nn}\end{bmatrix},\qquad
\mathbf{C}=
\begin{bmatrix}\mathbf{C}_{rr}&\mathbf{C}_{rn}\\ \mathbf{C}_{nr}&\mathbf{C}_{nn}\end{bmatrix}.
\end{equation*}
Assume $\mathbf{Q}_{nn}\succ 0$ on $\mathcal{N}$ and let $\{\mathbf{v}_j\}$ be an orthonormal eigenbasis of
$\mathbf{Q}_{nn}$ with $\mathbf{Q}_{nn}\mathbf{v}_j=\mu_j\,\mathbf{v}_j$ and $\mu_j>0$.
Define the $j$‑th null coefficient $a_j:=\mathbf{v}_j^\top \mathbf{x}_n$ where $\mathbf{x}_n:=\mathbf{P}_n\mathbf{x}$.
Consider measurements from \eqref{eq:inverse_problem}
and denote $\mathbf{C}_y:=\mathbf{H}\mathbf{C}_{rr}\mathbf{H}^\top+\sigma^2\mathbf{I}_m$.
Then the population $R^2$ of the optimal linear predictor of $a_j$ from $\mathbf{y}$ satisfies
\vspace{-0.2cm}
\begin{equation*}
\vspace{-0.2cm}
    \rho_j^2
:=\frac{\Cov(a_j,\mathbf{y})^\top\,\mathbf{C}_y^{-1}\,\Cov(\mathbf{y},a_j)}{\Var(a_j)}
\;\le\;
\frac{c_j}{c_j+\mu_j},
\end{equation*}
\[
c_j:=\mathbf{v}_j^\top\!\big(\mathbf{Q}_{rn}\mathbf{C}_{rr}\mathbf{Q}_{nr}\big)\mathbf{v}_j.
\]
\end{proposition}
\vspace{-0.2cm}
The proof of this proposition is found in Supp. \ref{app:proof_stats}.
\vspace{-0.1cm}

\begin{remark}
 If the prior is a GMRF with precision $\mathbf{Q}=\alpha\mathbf{L}+\varepsilon\mathbf{I}$ on $\mathbf{L}$, then $\mathbf{Q}_{nn}=\alpha\mathbf{T}+\varepsilon\mathbf{P}_n$ with
$\mathbf{T}:=\mathbf{P}_n\mathbf{L}\mathbf{P}_n$, and one can take $\{\mathbf{v}_j\}$ as the
$\mathbf{T}$–eigenmodes (graph‑smooth null modes). The bound becomes
\vspace{-0.2cm}
\begin{equation*}
\vspace{-0.2cm}
    \rho_j^2\;\le\;\frac{c_j}{c_j+\alpha\,\lambda_j(\mathbf{T})+\varepsilon},
\quad
c_j=\mathbf{v}_j^\top\big(\mathbf{Q}_{rn}\mathbf{C}_{rr}\mathbf{Q}_{nr}\big)\mathbf{v}_j,
\end{equation*}
with $\mathbf{Q}_{rn}=\alpha\mathbf{L}_{rn}$. Thus smoother null modes (small $\lambda_j(\mathbf{T})$) are more predictable. Now, if $\mathbf{L}=\mathbf{I}$ (isotropic prior), then $\mathbf{Q}=\gamma\mathbf{I}$ and
$\mathbf{Q}_{rn}=\mathbf{P}_r\,\mathbf{Q}\,\mathbf{P}_n=\mathbf{0}$, so $c_j=0$ and
$\rho_j^2\le 0$ for all $j$: there is \emph{no} statistical coupling between $\mathcal{R}$ and $\mathcal{N}$.
\end{remark}

\begin{figure}[!t]
    \centering
    \includegraphics[width=\linewidth]{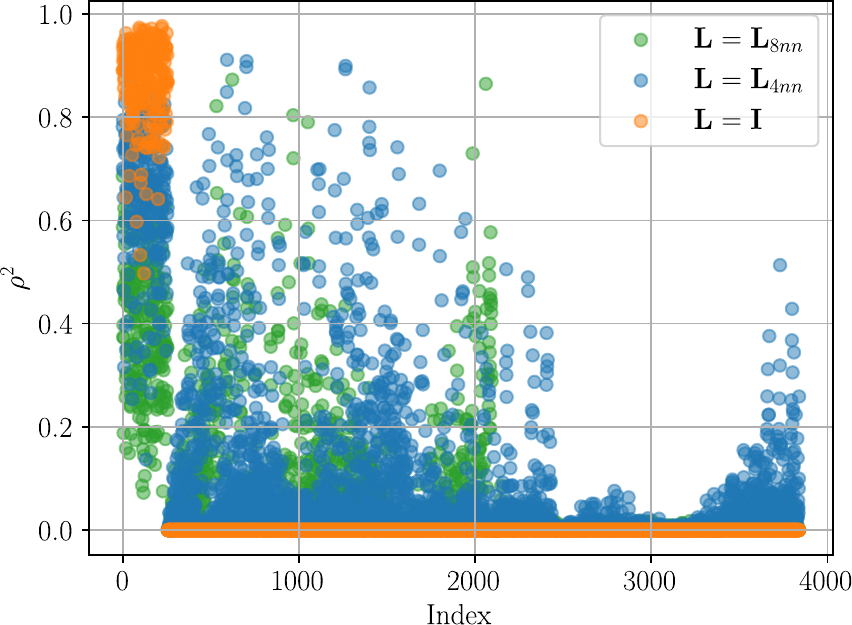}
    \vspace{-0.7cm}
    \caption{ Per-mode predictability for each case when $\mathbf{L} \in \{ \mathbf{I},\mathbf{L}_{8nn},\mathbf{L}_{4nn} \} $ for SR case with $SRF = 4$ with $n=64^2$.}
    \label{fig:predict}
\vspace{-0.6cm}
\end{figure}

\vspace{-0.2cm}
\noindent To evaluate the developed theory, Fig. \ref{fig:predict},  the
empirical per–mode predictability $\rho_j^2$ (one dot per eigenmode), for three choices of $\mathbf{L} \in \{ \mathbf{L}_{8nn},\mathbf{L}_{4nn},\mathbf{I} \} $, where $\mathbf{L=I}$ has same index ordering across panels. From indices $[0,m-1]$, the predictability is high for all settings, since they correspond to the eigenpairs related to the Range space, thus, the predictability is high.  The observations are:  (i) for $\mathbf{L=I}$, the spectrum is constant over the NS and then drops to $0$ outside of it, while both grid Laplacians exhibit a {smooth, monotone} growth of $\mu_j$;  (ii) correspondingly, with grid Laplacians, a {long tail} of modes has non‑zero predictability,
($\rho_j^2$ spread over a wide index range), whereas with $\mathbf{L=I}$ only a small initial block of modes
shows any predictability, and the rest are essentially $0$.  Note also that $\mathbf{L}_{4nn}$ has higher predictability than $\mathbf{L}_{8nn}$ in more eigenmodes, following the fast increase eigenvalue distribution Fig \ref{fig:graph_coverage} (c).

Note that this analysis was using a linear estimation framework, harnessing the NS GMRF prior. In practice, the \textit{statistical coupling} can also be achieved with the supervised learning of the predictor $\mathrm{G}$ in Eq. \eqref{eq:train-g}, but the trend of GS prior holds on this setting. To evaluate the predictability of the learned predictor, we used
\vspace{-0.3cm}
\begin{equation}
\vspace{-0.3cm}
    R^2(p) \triangleq 1-\frac{\mathbb{E}\left[\left\|\mathrm{G}^*(\mathbf{y})-\mathbf{S x}^{\ast}\right\|_2^2\right]}{\mathbb{E}\left[\left\|\mathbf{S x}^{\ast}\right\|_2^2\right]}.
\end{equation}

\subsection{Graph-based regularizer}
\label{sec:Graph-based_reg}

Consider the PnP–PGD iteration solving the optimization problem \eqref{eq:recon} without the term $\Vert \mathrm{G}^*(\mathbf{y})-\mathbf{Sx}\Vert$ (for more detailed analysis of this term, see \cite[Theorem 1]{Neurips}),   
\vspace{-0.6cm}

\begin{figure}[!t]
\centering
\includegraphics[width=\linewidth]{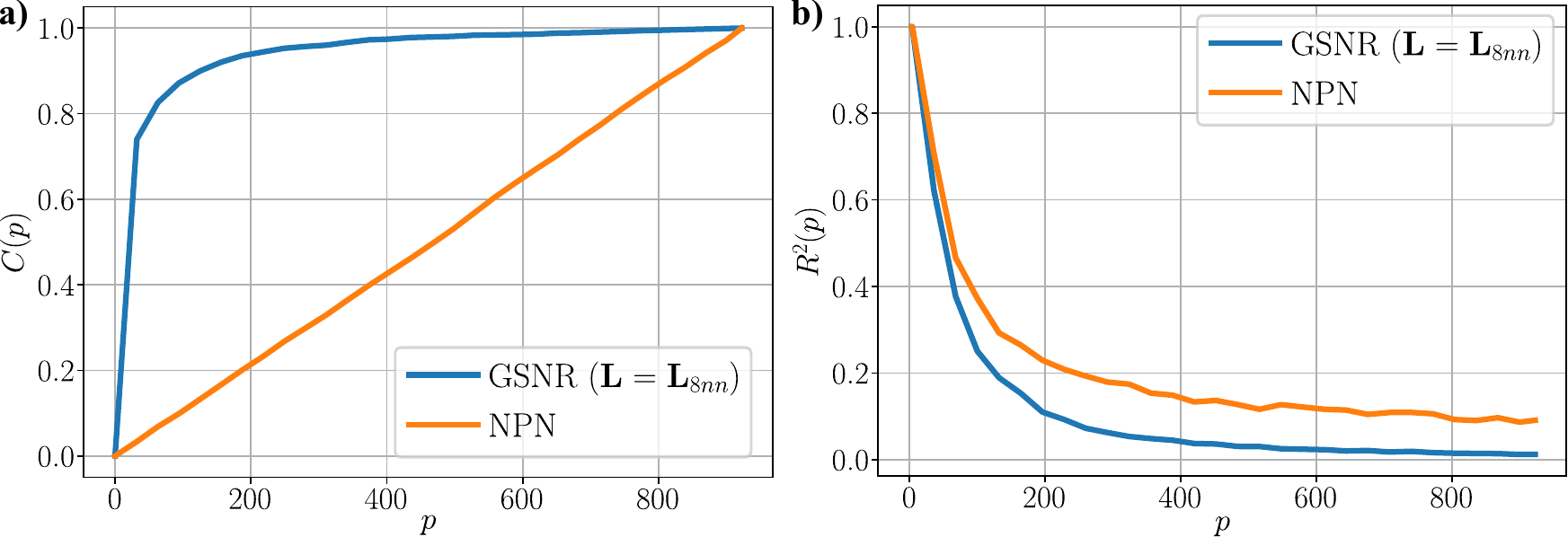}
\vspace{-0.7cm}
\caption{a) {Coverage} and b) Predictability vs. $p$ for CS with CIFAR-10 dataset. In this case, $m/n=0.1$ and $p = 1 \cdots n-m$.}
\label{fig:SPC_PredCov_vs_p}
\vspace{-0.6cm}
\end{figure}

\begin{figure*}[!t]
    \centering
\vspace{-0.3cm}
    \includegraphics[width=0.95\linewidth]{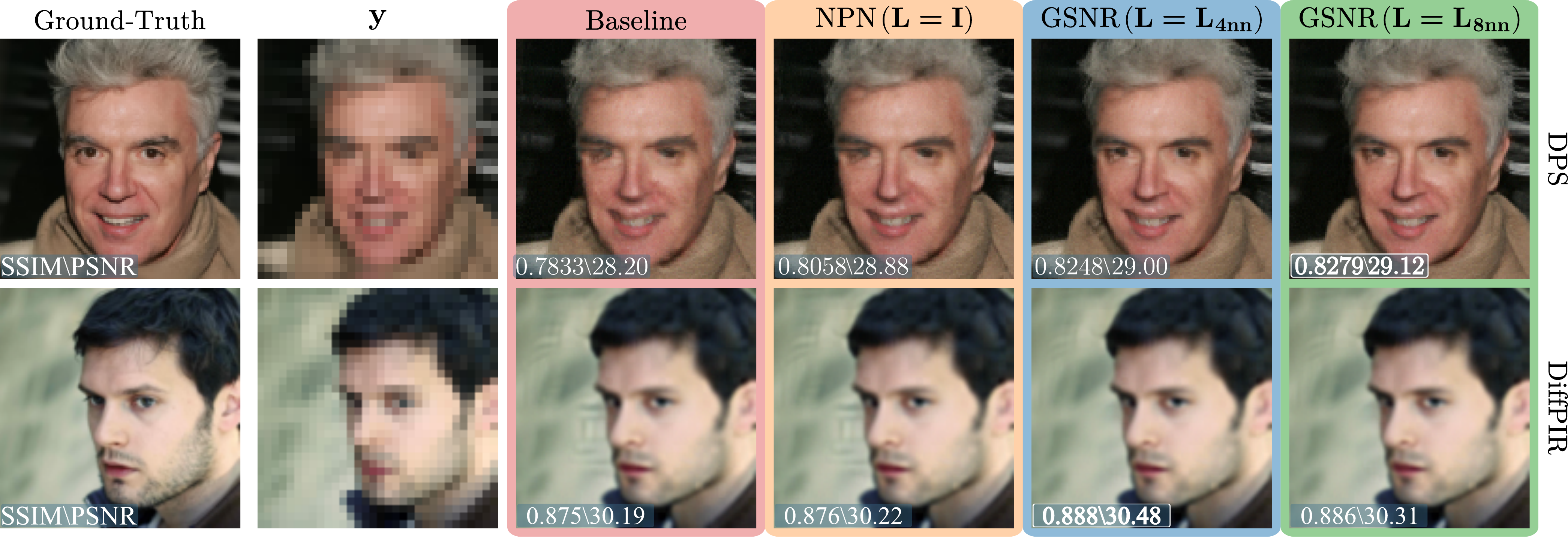}
\vspace{-0.3cm}
    \caption{Results of DM-based solvers (DPS \cite{dps} \& DiffPIR \cite{DiffPIR}) for Baseline, NPN \cite{Neurips}, and GSNR with $\mathbf{L_{4nn}}$ and $\mathbf{L_{8nn}}$. Here, $p=0.1n$.}
    \label{fig:dm}
\vspace{-0.55cm}
\end{figure*}

\begin{align}   
\vspace{-0.3cm}
\mathbf{x}_{k+1}&
= \mathrm{D}_\sigma\!\Big(\mathbf{x}_k-\alpha\big(\nabla g(\mathbf{x}_k)+\gamma_g\nabla \phi(\mathbf{x}_k)\big)\Big)\\&
= \mathrm{D}_\sigma\!\Big(\mathbf{x}_k-\alpha\big(\mathbf{H}^\top(\mathbf{Hx}_k-\mathbf{y})+\gamma_g\mathbf{Tx}_k\big)\Big),
\end{align}

\vspace{-0.2cm}
\noindent where $\mathrm{D}_\sigma$ is a bounded denoiser with $\|\mathrm{D}_\sigma(\mathbf{u})-\mathrm{D}_\sigma(\mathbf{v})\|\le\sqrt{1+\delta}\,\|\mathbf{u-v}\|$
(as in \cite[Assumption~2]{Neurips}), since $\nabla \phi(\mathbf{x}_k) = \mathbf{T}^\top \mathbf{x}_k = \mathbf{T} \mathbf{x}_k $. Let $\lambda_{\min}>0$ and $\lambda_{\max}$ be the extreme eigenvalues of $\mathbf{A}_{\gamma_g}=\mathbf{H^\top H+\gamma_g T}$ on
$\mathrm{span}(\mathrm{Range}(\mathbf{H}^\top)\cup\mathrm{Null}(\mathbf{H}))=\mathbb{R}^n$. Then for any step size $\alpha\in(0,\,2/\lambda_{\max})$,   
\vspace{-0.2cm}
\begin{align}   
\vspace{-0.2cm}
\|\mathbf{x}_{k+1}-\mathbf{x}^*\|
\;\le\;&
(1+\delta)\,\|\mathbf{I}-\alpha \mathbf{A}_{\gamma_g}\|_2\,\|\mathbf{x}_k-\mathbf{x}^*\|\\ \;\le\; &
\rho\,\|\mathbf{x}_k-\mathbf{x}^*\|,
\end{align}

\vspace{-0.25cm}
\noindent where $\rho=(1+\delta)\max\{\,1-\alpha\lambda_{\min},\,\alpha\lambda_{\max}-1\,\}$.
In particular, if ${\rho}<1$,
The iteration is a contraction and converges linearly to a fixed point.
The optimal $\alpha$ (for this linear bound) is $\alpha^*=\tfrac{2}{\lambda_{\min}+\lambda_{\max}}$
with rate
$\rho^*=(1+\delta)\,\frac{\kappa-1}{\kappa+1}$, $\kappa=\frac{\lambda_{\max}}{\lambda_{\min}}.$ Compared with the baseline ($\gamma_g=0$), adding $\gamma_g \mathbf{T}$: (i) makes $\mathbf{A}_{\gamma_g}\succ 0$ even on
$\mathrm{Null}(\mathbf{H})$, guaranteeing uniqueness of the quadratic subproblem, and (ii)
\emph{reduces} the condition number $\kappa$ because $\mathbf{T}$ acts precisely on directions where
$\mathbf{H}^\top \mathbf{H}$ is small, improving convergence and $\rho^*$.

\vspace{-0.2cm}
\section{Experiments}
\vspace{-0.2cm}

To validate the proposed approach, we implemented it in PyTorch~\cite{paszke2019pytorch} and used the DeepInverse library~\cite{tachella2025deepinverse} to benchmark different solvers, as detailed below. Since our work does not propose a specific network $\mathrm{G}$, we validate the approach using a U-Net model~\cite{unet} with a four-level 2D encoder--decoder architecture: the encoder applies double convolution blocks with $C, 2C, 4C,$ and $8C$ channels (with $C = 64$) and $2\times2$ max-pooling, while the decoder uses nearest-neighbor upsampling with skip connections, refining features back to $C$ channels before a final $1\times1$ convolution produces the output. The network is optimized using the Adam optimizer~\cite{adam} with a learning rate of $10^{-3}$ for 50 epochs. For all experiments, we used a noise variance $\sigma^2 = 0.05$. Code at \footnote{\href{http://github.com/yromariogh/GSNR}{github.com/yromariogh/GSNR}}. We consider the inverse problems:

\textbf{Compressed Sensing (CS):} $\mathbf{H} \in \{-1,1\}^{m \times n}$ are the first 10\% rows of a Hadamard matrix $\mathbf{A} \in \{-1,1\}^{n \times n}$~\cite{cs}. We validate with CIFAR-10 \cite{krizhevsky2009learning} and CelebA \cite{liu2015faceattributes} datasets. 

\textbf{Super-Resolution (SR):} We set $n = 3 \cdot 128^2$ and use an SR factor (SRF) of 4, i.e., $m = 3 \cdot 32^2$. We validate on the CelebA dataset~\cite{liu2015faceattributes}, using 10000 training images resized to $128\times128$. We also show results for the Synthetic Aperture Radar (SAR) image dataset~\cite{dalsasso2020sar}.

\textbf{Demosaicing:} We used the CelebA dataset, resized to $64\times 64$, and the Bayer filter \cite{bayer1976color} acquisition pattern.

\textbf{Deblurring:} We set $n = 64^2 \cdot 3$ and used a 2-D Gaussian kernel with a bandwidth $\sigma_k = 1$. We validated our experiments on the CelebA \cite{liu2015faceattributes} dataset and the Places365 \cite{zhou2017places} dataset, both resized to $64\times 64$, using 10000 images for training and a batch size of 32. %

\noindent For all of these imaging tasks, we construct the null-restricted Laplacian $\mathbf{T}$ and the matrix $\mathbf{S}$, leveraging properties of $\mathbf{H}$. Details on the computation of these matrices in \ref{app:t_constructio}. We provide comprehensive ablations of GSNR in \ref{app:graph_reg}.

\vspace{-0.2cm}
\subsection{Predictability and Coverage} 
\vspace{-0.2cm}

The coverage definition in \eqref{eq:Cp}, has a general expression for the covariance
matrix computation. In practice, we computed as follows. Given centered images
$\mathbf{x}_i$ and their null projections $\mathbf{x}_{n,i} = \mathbf{P}_n (\mathbf{x}_i - \bar{\mathbf{x}})$, where
$\bar{\mathbf{x}} = \frac{1}{N} \sum_{i=1}^N \mathbf{x}_i$ and the NS covariance $\mathrm{Cov}(\mathbf{x}_n) = \frac{1}{N} \sum_{i=1}^N \mathbf{x}_{n,i} \mathbf{x}_{n,i}^\top$. Fig. \ref{fig:SPC_PredCov_vs_p} shows the results in terms of (a) Coverage and (b) Predictability in CS for different operators $\mathbf{S}$, under the same acquisition matrix $\mathbf{H}$. We used the $N=10000$ image of the CIFAR-10 dataset. In this case, NPN ($\mathbf{L=I}$) is obtained following \cite[Eq. 3]{Neurips}, and GSNR uses $\mathbf{L=L}_{8nn}$. Note that here we take the first $p$ rows of each $\mathbf{S}$. Supp. \ref{app:cs_results} shows additional results for CS. \ref{app:coverage_in_other_tasks} shows the coverage curves for the SR task.

\vspace{-0.2cm}
\subsection{Diffusion Solvers}
\vspace{-0.2cm}

We used two diffusion model (DM) solvers, DPS \cite{dps} and DiffPIR \cite{DiffPIR}. See Supp. \ref{app:GSNR_in_DM} for more details on the adaptation of these models with the proposed approach. Fig.~\ref{fig:dm} compares DPS and DiffPIR with and without the proposed GSNR term. In all cases, the baseline DM already produces plausible faces, but it often leaves residual blur and aliasing around high–frequency structures such as hair, eyebrows, and facial contours. When we replace the generic NS basis by the GSNR, the reconstructions become visibly sharper and more faithful to the ground truth: edges around the jawline and nose are better defined, textures are less washed out, and blocky artifacts from the low-resolution observation $\mathbf{y}$ are further suppressed. These qualitative improvements are consistent with the quantitative SSIM/PSNR numbers printed in each panel: the gains over the baseline DiffPIR/DPS are modest in dB but systematic (e.g., $30.19\to 30.31$; $30.19\to 30.48$), indicating that the GSNR term helps the diffusion prior resolve NS ambiguities rather than hallucinating arbitrary details. In other words, the DM still provides the powerful image prior, but GSNR steers it along graph-smooth NS directions that are predictable from $\mathbf{y}$ and compatible with the forward model, yielding more accurate and stable reconstructions. Supp. \ref{app:GSNR_in_DM} shows the integration of GSNR in latent-space DMs \cite{he2024manifold}.

\begin{table}[!t]
\centering
\caption{Quantitative comparison on the super-resolution and demosaicing tasks (higher is better) using the CelebA test set. For each task, the \best{best} result and \second{second-best} results.}
\vspace{-0.3cm}
\label{tab:sr_demosaic}
\resizebox{\linewidth}{!}{%
\begin{tabular}{lcc}
\toprule
\multirow{2}{*}{Method} & \multicolumn{2}{c}{Task} \\
\cmidrule(lr){2-3}
 & Super-Resolution & Demosaicing \\
\midrule
PnP \cite{kamilov2023plug}          & 27.37           & 39.35 \\
NPN \cite{Neurips}                                     & 29.21           & 39.77 \\
DDN-Cascade \cite{chen2020deep}         & 28.92           & 37.83 \\
DDN-Independent \cite{chen2020deep}     & 28.81           & 38.92 \\
DNSN \cite{schwab2019deep}              & 26.52           & 39.33 \\
Unrolling \cite{xiang2021fista}         & 28.97           & 38.82 \\
GSNR-PnP w.\ $\mathbf{L}_{8\text{nn}}$  & \second{29.38}  & \second{39.88} \\
GSNR-PnP w.\ $\mathbf{L}_{4\text{nn}}$  & \best{29.42}    & \best{39.89} \\
\bottomrule
\end{tabular}%
}
\vspace{-0.6cm}
\end{table}

\vspace{-0.3cm}
\subsection{Comparison with End-to-End Methods}
\vspace{-0.2cm}

Table~\ref{tab:sr_demosaic} reports PSNR on SR and demosaicing, comparing classical PnP, the proposed GSNR-PnP variants, NPN, and several strong deep-learning baselines. Across both tasks, GSNR-PnP with graph Laplacians $\mathbf{L}_{4\text{nn}}$ and $\mathbf{L}_{8\text{nn}}$ achieves the best performance. For SR, GSNR-PnP yields roughly a $2$~dB gain over vanilla PGD-PnP and a clear margin over both NPN (with $\mathbf{L=I}$) and task-specific unrolled/cascade architectures. For demosaicing, the problem is less ill-posed, so the absolute gains are smaller, but GSNR-PnP still provides a consistent improvement over PGD-PnP, NPN, and the deep baselines. Overall, these results show that injecting graph-smooth NS structure into a standard PnP solver consistently improves upon (i) geometry-free PnP/NPN baselines and (ii) specialized deep unrolled or cascade architectures, with particularly strong benefits on the most underdetermined SR task. In \ref{app:demosaicking} and \ref{app:sr_results}, we show additional results for Demosaicking and SR, respectively.

\begin{table}[t]
\caption{Final PSNR (dB) for deblurring ($\sigma_k=1$) with $n=64^2\cdot 3$ and $p=0.8\cdot n$.
We report results on Places365~\cite{dalsasso2020sar} and CelebA~\cite{liu2015faceattributes} test sets
using Lip-DnCNN~\cite{ryu2019plug} and a wavelet denoiser~\cite{donoho2002noising}. For each column, \best{best} and \second{second-best} results.}   
\vspace{-0.3cm}
\label{tab:psnr_compact_deblur}
\centering
\resizebox{\linewidth}{!}{%
\begin{tabular}{lcccccc}
\toprule
\multirow{2}{*}{Method} & \multirow{2}{*}{$\gamma_{\mathrm{g}}$}
 & \multicolumn{2}{c}{Places365} & \multicolumn{2}{c}{CelebA} \\
\cmidrule(lr){3-4} \cmidrule(lr){5-6}
 & & Lip-DnCNN & Wavelet & Lip-DnCNN & Wavelet \\
\midrule
Baseline (PGD-PnP)      & --        & 31.58 & 30.78 & 35.26 & 33.64 \\
NPN~\cite{Neurips}      & --        & 33.22 & 33.17 & 37.86 & 37.68 \\
GSNR w.\ $\mathbf{L}_{8\text{nn}}$ & \ding{56} & \second{33.60} & \second{33.59} & \second{38.18} & \second{37.99} \\
GSNR w.\ $\mathbf{L}_{8\text{nn}}$ & \ding{51} & \best{33.60} & \best{33.59} & \best{38.18} & \best{37.99} \\
\bottomrule
\end{tabular}%
}
\vspace{-0.7cm}
\end{table}

\vspace{-0.2cm}
\subsection{Plug-and-Play}
\vspace{-0.2cm}

We used the PGD-PnP \cite{hurault2023relaxed, kamilov2023plug}, but GSNR can be easily adapted to other formulations, such as ADMM \cite{chan2016plug} or HQS \cite{rasti2023plug}.  We used three denoisers: DnCNN with Lipchitz training (Lip-DnCNN) \cite{ryu2019plug}, DRUNet \cite{zhang2021plug}, Wavelet denoiser (Daubechies 8 with 4 levels) using soft-thresholding proximal \cite{donoho2002noising}. For improving stability and performance of the PGD-PnP, we include the equivariant denoising technique \cite{terris2024equivariant}  with two random $90°$ rotations and two reflections.


\textbf{Deblurring}: Table~\ref{tab:psnr_compact_deblur} reports the final PSNR for deblurring on Places365 and CelebA under different graph variants and denoisers. In all settings, both NPN and GSNR substantially improve over the baseline PGD-PnP solver: on CelebA, the gain is on the order of $+2$~dB for both Lip-DnCNN and the wavelet denoiser, and on Places365, the gain is around $+1.5$–$2$~dB. The GSNR variant with $\mathbf{L}_{8\text{nn}}$ provides a small but consistent improvement over NPN on both datasets and for both denoisers, indicating that injecting NS structure into the NS helps the solver recover sharper details without
overfitting. Interestingly, turning the additional graph penalty on or off
($\gamma_{\mathrm{g}}$ \ding{51} vs.\ \ding{56}) has a negligible effect on the final PSNR in this experiment, suggesting that the learned GSNR already captures most of the useful graph structure for deblurring. Figure~\ref{fig:conv_db} shows the deblurring convergence behaviour with the DnCNN denoiser. The baseline PnP-PGD curve increases steadily but saturates at a lower PSNR plateau. In contrast, both GSNR
variants (with $\gamma_{\mathrm{g}}=0$ and $\gamma_{\mathrm{g}}=0.1$) converge to a noticeably
higher final PSNR, confirming that the graph-smooth null-space representation improves the
fixed point of the algorithm. The explicit null-only graph penalty ($\gamma_{\mathrm{g}}=0.1$)
mainly affects the {transient} regime: it accelerates convergence in the first few hundred iterations, reaching near-peak PSNR significantly earlier than GSNR with $\gamma_{\mathrm{g}}=0$ \textcolor{black}{(aligned with the theoretical analysis conducted in Sec. \ref{sec:Graph-based_reg})}, while all three methods eventually stabilize. Additional visual results in \ref{app:deblurring}.
\begin{figure}
    \centering
    \includegraphics[width=0.85\linewidth]{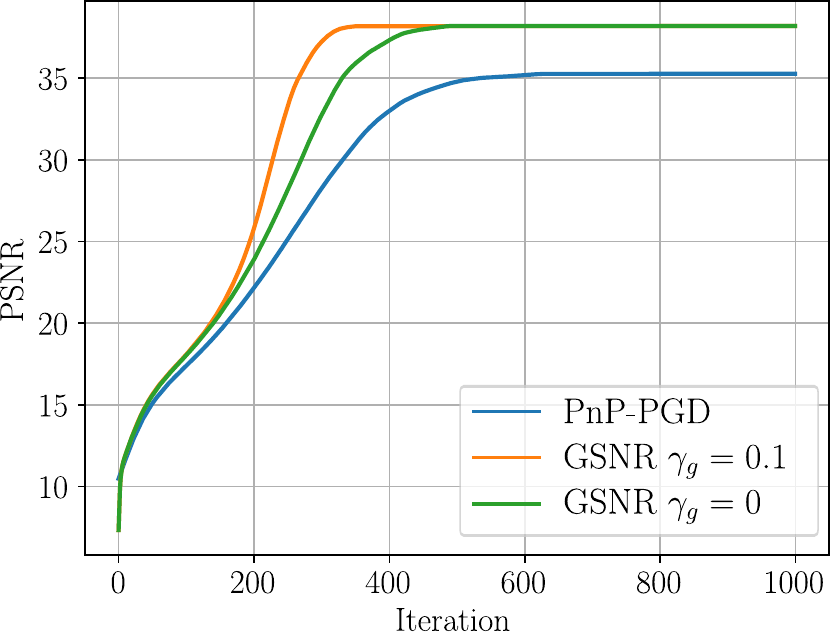}
    \vspace{-0.3cm}
    \caption{Deblurring convergence curves for PnP-PGD baseline using Lip-DnCNN and GSNR variants with $\mathbf{L}_{8nn} $ with and without the null-restricted Laplacian regularization.}\vspace{-0.74cm}
    \label{fig:conv_db}
\end{figure}

\textbf{SR}: Fig.~\ref{fig:sr_single_example} shows an SR example for different graph designs and denoisers. Columns correspond to the choice of Laplacian $\mathbf{L}$ or to the baseline PnP without GSNR, while rows correspond to different denoisers. In all three rows, moving from the baseline and $\mathbf{L}=\mathbf{I}$ columns to the graph-based $\mathbf{L}_{4\mathrm{nn}}$ and $\mathbf{L}_{8\mathrm{nn}}$ columns produces visibly sharper details: facial contours are crisper, hair strands are better resolved, and blocky artifacts inherited from the low-resolution backprojection $\mathbf{H}^\top\mathbf{y}$ are reduced. The improvements are most pronounced for the DnCNN and DRUNet denoisers. Still, they are also clearly present for the simpler Wavelet denoiser, indicating that the benefit comes from the GSNR itself rather than from the denoiser. Among the graph designs, $\mathbf{L}_{8\mathrm{nn}}$ yields the visually clearest reconstruction, suggesting that richer graph connectivity enables the solver to exploit structure in the null space rather than hallucinating high-frequency content. Results with Deep Image Prior \cite{deep_prior} and SAR images in SR can be found in \ref{app:sr_results}.\vspace{-0.1cm}


\vspace{-0.2cm}
\section{Discussion}
\vspace{-0.2cm}
We demonstrate that GSNR is a versatile approach for a broad range of inverse problems, data modalities, and neural networks (\ref{app:neural_network_ablation}).  
Nevertheless, the method requires precise knowledge of the sensing matrix, which is unavailable in some applications, e.g., real-world image deblurring. In \ref{app:inexact_forward_operator}, we provide analysis on the effect of uncertainty in the sensing operator, where a performance drop is evident, but the improvement compared with the baseline algorithm is maintained. Follow-up ideas will harness the null-space construction under sensing matrix uncertainty for calibrated null-space {regularization functions}. Another aspect of GSNR is the requirement to perform EVD on the null-restricted Laplacian $\mathbf{T}$, which, at large image scales, necessitates substantial computational resources. However, this computation is performed \textit{offline} once per $(\mathbf{H, L},n,p)$, and reused for learning the network $\mathrm{G}$ and the reconstruction step. In \ref{app:scalability}, we provide a detailed runtime analysis of the approach. Additionally, our approach is designed for linear inverse problems; the adaptation for nonlinear inverse problems, e.g., phase retrieval, requires careful modeling of the invertibility properties of the sensing operator and its input-dependent ambiguities. In recent work \cite{ehrlich2026pseudo}, a generalization of the null-space property was developed for neural networks (non-linear functions), opening new frontiers for learning optimal null-space components.

\begin{figure}[!t]
    \centering
    \includegraphics[width=\linewidth]{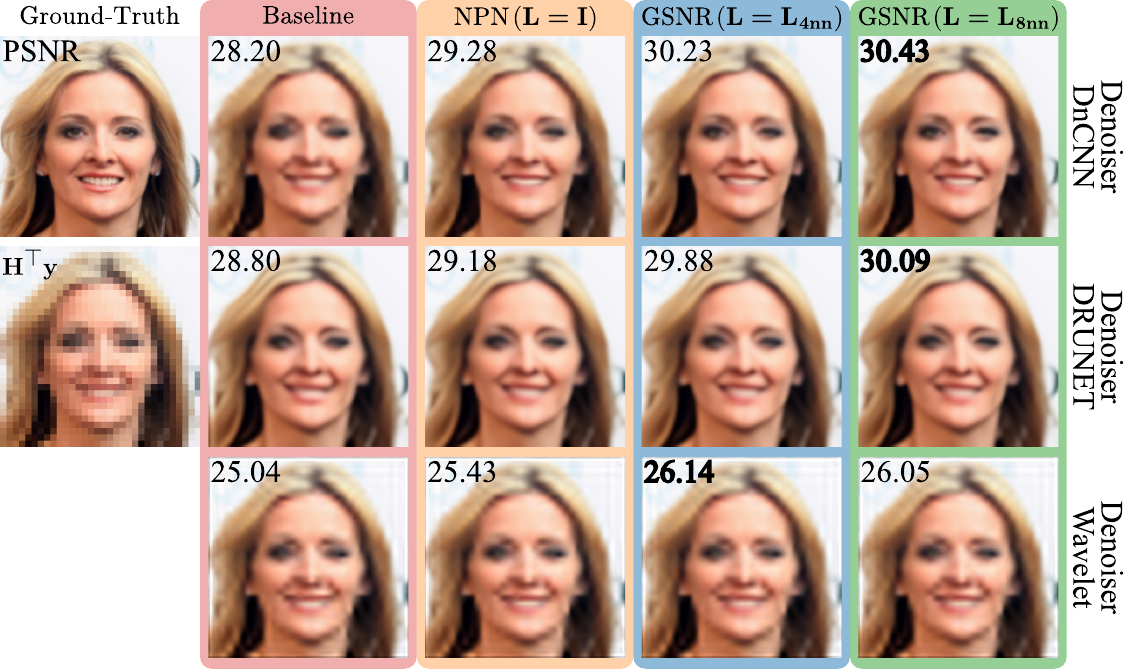}
    \vspace{-0.7cm}
    \caption{Results of PnP, NPN, and GSNR varying $\mathbf{L}$ and denoisers. Best results for each denoiser are in bold. Here, $p=0.1n$.}
    \vspace{-0.65cm}
    \label{fig:sr_single_example}
\end{figure}

\vspace{-0.2cm}

\section{Conclusion and Future Outlook}

\vspace{-0.1cm}

We introduced \emph{GSNR: Graph-Smooth Null-Space Representation}, which equips an inverse problem solver not only with \emph{what} the measurements constrain, but also with \emph{how} to use the degrees of freedom that the sensor never sees (\textcolor{black}{null-space}). Instead of regularizing the full image, GSNR builds a low-dimensional basis for $\Null(\mathbf{H})$ aligned with graph smoothness via the null-restricted Laplacian $\mathbf{T} = \mathbf{P}_n \mathbf{L} \mathbf{P}_n$, and can optionally add a null-only graph penalty at inference. Theoretically, the smoothest $\mathbf{T}$-modes are minimax-optimal for null-space components, and under a GMRF prior, they provably capture more null-space variance than geometry-free choices while enjoying stronger statistical coupling to the measurements. In practice, this yields better coverage of plausible solutions, higher predictability of null-space signal components, and improved conditioning, translating into faster and more stable convergence of PnP and diffusion-based solvers. Looking ahead, rather than fixing $\mathbf{L}$ to a hand-crafted grid Laplacian, learning the graph structure or employing interpretable graph neural operators promises data-adaptive GSNR bases that better reflect semantic structure. Also, a more rigorous study of GSNR within generative solvers could provide a principled way to inject graph-smooth null constraints into these frameworks without breaking their probabilistic interpretation.

%% file: sec/X_suppl.tex
\clearpage
\setcounter{page}{1}
\maketitlesupplementary


\appendix

\setcounter{section}{0}              
\renewcommand\thesection{A\arabic{section}}

\addcontentsline{toc}{section}{Appendix}

\startcontents[appendices]
\printcontents[appendices]{}{1}{\section*{Appendix Contents}}
\section{Proof of Theorem \ref{th:coverage}} \label{app:proof_cov}
\setcounter{theorem}{0}
\begin{theorem}[Coverage for graph smooth null-space]
Consider the construction of $\mathbf{T}$ Eq \eqref{eq:T_definition}, the covariance of the null-space $\mathrm{Cov}(\mathbf{x}_n)$  is a spectral function of $\mathbf{T}$, i.e.,  $\mathrm{Cov}(\mathbf{x}_n) = \mathbf{V}\mathrm{diag}({\lambda_1,\dots, \lambda_q})\mathbf{V}^\top$, where $\lambda_i = \frac{1}{\alpha \mu_i + \epsilon}$, $\lambda_1\geq\lambda_2\cdots\geq \lambda_q$. With the construction of $\mathbf{S}$ in \eqref{eq:S_definition}, the coverage of the null-space using GMRF with a Laplacian matrix $\mathbf{L}$ as $C_L(p)$ and the coverage when $\mathbf{L=I}$ is denoted $C_I(p)$, satisfies for every $p=1,\dots, q$

\begin{equation}
    C_L(p)\ \ge\ C_I(p)
\end{equation}
\end{theorem}

\noindent Using cyclicity of the trace and $P(p)=\mathbf{V}_p\mathbf{V}_p^\top$,
\begin{align*}
\mathrm{tr}\big(P(p)\,\mathrm{Cov}\big)
&= \mathrm{tr}\Big(\mathbf{V}_p\mathbf{V}_p^\top \,\mathbf{V}\,\mathrm{diag}(\lambda)\,\mathbf{V}^\top\Big)\\&
= \mathrm{tr}\Big( (\mathbf{V}^\top \mathbf{V}_p)(\mathbf{V}^\top \mathbf{V}_p)^\top\,\mathrm{diag}(\lambda)\Big) \\
&= \mathrm{tr}\big(\mathrm{diag}(\lambda_1,\ldots,\lambda_p)\big)
= \sum_{i=1}^p \lambda_i.
\end{align*}
Also $\mathrm{tr}(\mathrm{Cov})=\sum_{i=1}^q \lambda_i$. Hence
\begin{equation}
C_L(p) \;=\; \frac{\sum_{i=1}^{p}\lambda_i}{\sum_{i=1}^{q}\lambda_i}.
\label{eq:CL_closed_form}
\end{equation}

 Consider the case $\mathbf{L=I} \xrightarrow{} \mathbf{T} = \mathbf{P}_n$ due to the symmetric and idempotent property of the null-space projector. Any symmetric and idempotent matrix is diagonalizable with eigenvalues only in $\{0,1\}$. Then, any eigenvector and eigenvalue pair $(\mathbf{v}_i,\lambda_i)$ satisfy
\[
\lambda_i =
\begin{cases}
1, & \text{if } \mathbf{v}_i\in \mathrm{Null}(\mathbf{H}),\\[4pt]
0, & \text{if } \mathbf{v}_i\in \mathrm{Range}(\mathbf{H}^\top).
\end{cases}
\]

\noindent Under this analysis, replacing the eigenvalues $\lambda_i$ in \eqref{eq:CL_closed_form}, we have 

\begin{equation}
    C_I(p) = \frac{p}{q}. \label{eq:CI_linear}
\end{equation}

\noindent Since $(\lambda_i)$ is nonincreasing, the mean of the first $p$ terms is at least the
global mean:
\[
\frac{1}{p}\sum_{i=1}^p \lambda_i \;\ge\; \frac{1}{q}\sum_{i=1}^q \lambda_i
\quad\Longleftrightarrow\quad
\frac{\sum_{i=1}^p \lambda_i}{\sum_{i=1}^q \lambda_i} \;\ge\; \frac{p}{q}.
\]
Combining with \eqref{eq:CL_closed_form}–\eqref{eq:CI_linear} gives
$C_L(p)\ge C_I(p)$ for all $p$, with strict inequality whenever $(\lambda_i)$ is not constant.
\setcounter{theorem}{0}
\begin{corollary}
Using $\lambda_i=\frac{1}{\alpha\mu_i+\varepsilon}$, $\mu_1\le\cdots\le \mu_q$, and the
bounds $\sum_{i\le p}\lambda_i\ge p\,\lambda_p$ and
$\sum_{i\le q}\lambda_i\le p\,\lambda_1 + (q-p)\,\lambda_{p+1}$, we obtain
\begin{align}
    C_L(p)&\ge
\frac{p\,\lambda_p}{p\,\lambda_1+(q-p)\,\lambda_{p+1}}\\&
=\frac{1}{\ \frac{\alpha\mu_p+\varepsilon}{\alpha\mu_1+\varepsilon}
+\frac{q-p}{p}\,\frac{\alpha\mu_p+\varepsilon}{\alpha\mu_{p+1}+\varepsilon}\ } \;\ge\; \frac{p}{q},
\end{align}
with the last inequality again strict unless the spectrum is flat.
This makes explicit how fast coverage rises for Laplacians whose
$\{\mu_i\}$ increase rapidly (grids/graphs) compared to $\mathbf{L=I}$.
\end{corollary}

\section{Proof of Theorem \ref{th:Th1}} 
\label{app:minmax_thm}
\begin{theorem}[Minimax optimality]
Let \(V_p=\mathrm{span}\{\mathbf{v}_1,\dots,\mathbf{v}_p\}\) and \(P_{V_p}\) be its projector. Then, among all \(p\)-dimensional \(V\subset\mathrm{Null}(\mathbf{H})\), 
provided \(\mu_{p+1}>0\),
\vspace{-0.1cm}
\begin{align*}
\vspace{-0.1cm}
\min_{\dim(V)=p}\;
    &\sup_{\mathbf{x}_n \in \mathcal{M}_\tau}\;
    \|(\mathbf{I} - P_V)\mathbf{x}_n\|_2^2 = \\
    &\sup_{\mathbf{x}_n \in \mathcal{M}_\tau}\;
    \|(\mathbf{I} - P_{V_p})\mathbf{x}_n\|_2^2 = \frac{\tau}{\mu_{p+1}}.
\end{align*}
\end{theorem}

\noindent \textit{\textbf{Proof.}}

\noindent Expand any $\mathbf{x}_n\in\Null(\mathbf{H})$ in the eigenbasis of $\mathbf{T}$:
$\mathbf{x}_n=\sum_{j=1}^q a_j \mathbf{v}_j$ with $\|\mathbf{x}_n\|_2^2=\sum_j a_j^2$
and energy constraint $\mathbf{x}_n^\top\mathbf{T}\mathbf{x}_n=\sum_j \mu_j a_j^2 \le \tau$.
This is the Rayleigh–Ritz parameterization for a symmetric PSD matrix.  Take $V_p=\mathrm{span}\{\mathbf{v}_1,\dots,\mathbf{v}_p\}$.
Then $(\mathbf{I}-P_{V_p})\mathbf{x}_n=\sum_{j\ge p+1} a_j \mathbf{v}_j$, hence
\[
\|(\mathbf{I}-P_{V_p})\mathbf{x}_n\|_2^2=\sum_{j\ge p+1} a_j^2
\;\le\;\frac{1}{\mu_{p+1}}\sum_{j\ge p+1}\mu_j a_j^2
\;\le\;\frac{\tau}{\mu_{p+1}}.
\]
The first inequality uses $\mu_j\ge\mu_{p+1}$ for all $j\ge p+1$.
Equality is attained by $\mathbf{x}_n^\star=\sqrt{\tau/\mu_{p+1}}\,\mathbf{v}_{p+1}$,
so
\[
\sup_{\mathbf{x}_n\in\mathcal{M}_\tau}\|(\mathbf{I}-P_{V_p})\mathbf{x}_n\|_2^2=\frac{\tau}{\mu_{p+1}}.
\tag{$\ast$}
\]

\medskip
Let $V$ be any $p$–dimensional subspace of $\Null(\mathbf{H})$ and set $W:=V^\perp$
(within $\Null(\mathbf{H})$), so $\dim W=q-p$.
Consider the restricted eigenproblem for $\mathbf{T}$ on $W$; denote its smallest eigenvalue by
\[
\widetilde{\mu}_{\min}(W) \;=\;\min_{\mathbf{u}\in W,\ \|\mathbf{u}\|=1}\ \mathbf{u}^\top \mathbf{T}\mathbf{u}.
\]
By the Courant–Fischer min–max theorem,
\begin{align}
\mu_{p+1} \;=&\; \max_{\dim(S)=q-p}\ \min_{\mathbf{u}\in S,\ \|\mathbf{u}\|=1}\ \mathbf{u}^\top\mathbf{T}\mathbf{u}
\nonumber\\\;\ge&\; \min_{\mathbf{u}\in W,\ \|\mathbf{u}\|=1}\ \mathbf{u}^\top\mathbf{T}\mathbf{u}
\;=\; \widetilde{\mu}_{\min}(W).
\end{align}
Thus $\widetilde{\mu}_{\min}(W)\le \mu_{p+1}$.  

Now maximize the residual under the energy constraint \emph{within $W$}, where the projection
vanishes: for any $\mathbf{x}_n\in W$, $(\mathbf{I}-P_V)\mathbf{x}_n=\mathbf{x}_n$.
The maximizer aligns with the eigenvector of $\mathbf{T}|_W$ for $\widetilde{\mu}_{\min}(W)$, yielding
\begin{align}
\sup_{\mathbf{x}_n\in\mathcal{M}_\tau}\|(\mathbf{I}-P_V)\mathbf{x}_n\|_2^2
&\;\ge\;\sup_{\substack{\mathbf{x}_n\in W\\ \mathbf{x}_n^\top\mathbf{T}\mathbf{x}_n\le\tau}} \|\mathbf{x}_n\|_2^2
\;=\;\frac{\tau}{\widetilde{\mu}_{\min}(W)}
\nonumber\\&\;\ge\;\frac{\tau}{\mu_{p+1}}.
\tag{$\dagger$}
\end{align}
Combining $(\ast)$ and $(\dagger)$ gives
\[
\min_{\dim(V)=p}\ \sup_{\mathbf{x}_n\in\mathcal{M}_\tau}\ \|(\mathbf{I}-P_V)\mathbf{x}_n\|_2^2
\;=\; \frac{\tau}{\mu_{p+1}},
\]
with the minimum attained at $V=V_p$.

\begin{remark}
 The constraint
$\mathbf{x}_n^\top\mathbf{T}\mathbf{x}_n\le\tau$ is an ellipsoid aligned with
the eigenvectors of $\mathbf{T}$. The largest Euclidean norm inside this ellipsoid occurs along
the \emph{smallest} eigenvalue direction. If the subspace $V$ leaves a small-$\mu$ direction
outside, an adversary can place all energy there and produce a large residual. Courant–Fischer
formalizes this: every $p$–dimensional $V$ leaves some direction with Rayleigh quotient
$\le\mu_{p+1}$ in $V^\perp$.  

Selecting
$V_p=\mathrm{span}\{\mathbf{v}_1,\dots,\mathbf{v}_p\}$ “covers” all directions with the
smallest graph energy (smoothest modes). The worst direction you do \emph{not} cover is then
$\mathbf{v}_{p+1}$ with energy $\mu_{p+1}$, making the worst-case miss exactly $\tau/\mu_{p+1}$.
Any other choice leaves an even {smoother} (smaller $\mu$) direction uncovered, increasing
the worst-case error.

The bound quantifies how quickly the worst-case miss decays as
$p$ grows: the key driver is $\mu_{p+1}$. For graph Laplacians, $\{\mu_j\}$ increases smoothly,
so $\mu_{p+1}$ grows and the miss shrinks rapidly; for $\mathbf{L}=\mathbf{I}$, $\mathbf{T}=\mathbf{P}_n$
has a flat spectrum on $\Null(\mathbf{H})$, so $\mu_{p+1}$ is constant and the bound does not improve.
This explains why graph-limited designs achieve much better null-space coverage than geometry-free design
choices (the argument is the Kolmogorov $n$‑width of the ellipsoid $\{\mathbf{x}_n:\mathbf{x}_n^\top
\mathbf{T}\mathbf{x}_n\le\tau\}$). 
\end{remark}
\setcounter{proposition}{0}

\section{Proof of Proposition \ref{thm:modeR2-bold}}
\label{app:proof_stats}
\begin{proposition}[Per–mode predictability bound]
Let $\mathbf{x}\!\in\!\mathbb{R}^n$ be zero–mean Gaussian with covariance $\mathbf{C}$ and precision
$\mathbf{Q}=\mathbf{C}^{-1}$. Let $\mathbf{H}\!\in\!\mathbb{R}^{m\times n}$ and denote the orthogonal
projectors onto $\mathcal{R}:=\Range(\mathbf{H}^\top)$ and $\mathcal{N}:=\Null(\mathbf{H})$ by
$\mathbf{P}_r$ and
$\mathbf{P}_n$, respectively.
Block $\mathbf{Q}$ and $\mathbf{C}$ with respect to the decomposition $\mathbb{R}^n=\mathcal{R}\oplus\mathcal{N}$:
\vspace{-0.1cm}
\begin{equation*}
\vspace{-0.1cm}
    \mathbf{Q}=
\begin{bmatrix}\mathbf{Q}_{rr}&\mathbf{Q}_{rn}\\ \mathbf{Q}_{nr}&\mathbf{Q}_{nn}\end{bmatrix},\qquad
\mathbf{C}=
\begin{bmatrix}\mathbf{C}_{rr}&\mathbf{C}_{rn}\\ \mathbf{C}_{nr}&\mathbf{C}_{nn}\end{bmatrix}.
\end{equation*}
Assume $\mathbf{Q}_{nn}\succ 0$ on $\mathcal{N}$ and let $\{\mathbf{v}_j\}$ be an orthonormal eigenbasis of
$\mathbf{Q}_{nn}$ with $\mathbf{Q}_{nn}\mathbf{v}_j=\mu_j\,\mathbf{v}_j$ and $\mu_j>0$.
Define the $j$‑th null coefficient $a_j:=\mathbf{v}_j^\top \mathbf{x}_n$ where $\mathbf{x}_n:=\mathbf{P}_n\mathbf{x}$.
Consider measurements from \eqref{eq:inverse_problem}
and denote $\mathbf{C}_y:=\mathbf{H}\mathbf{C}_{rr}\mathbf{H}^\top+\sigma^2\mathbf{I}_m$.
Then the population $R^2$ of the optimal linear predictor of $a_j$ from $\mathbf{y}$ satisfies
\vspace{-0.1cm}
\begin{equation*}
\vspace{-0.1cm}
    \rho_j^2
:=\frac{\Cov(a_j,\mathbf{y})^\top\,\mathbf{C}_y^{-1}\,\Cov(\mathbf{y},a_j)}{\Var(a_j)}
\;\le\;
\frac{c_j}{c_j+\mu_j},
\end{equation*}
\[
c_j:=\mathbf{v}_j^\top\!\big(\mathbf{Q}_{rn}\mathbf{C}_{rr}\mathbf{Q}_{nr}\big)\mathbf{v}_j.
\]
Equality holds in the ideal case $\mathbf{H}=\mathbf{I}$ and $\sigma^2=0$.
\end{proposition}

Since $\mathbf{Q}\succ 0$ and $\mathbf{Q}_{nn}\succ 0$, the block inverse of $\mathbf{Q}$ exists and yields
the standard formulas (Schur complements)
\begin{align}\label{eq:block-identities}
&\mathbf{C}_{nr}=-\,\mathbf{Q}_{nn}^{-1}\mathbf{Q}_{nr}\mathbf{C}_{rr},\nonumber\\&
\mathbf{C}_{nn}=\mathbf{Q}_{nn}^{-1}+\mathbf{Q}_{nn}^{-1}\mathbf{Q}_{nr}\mathbf{C}_{rr}\mathbf{Q}_{rn}\mathbf{Q}_{nn}^{-1}.\end{align}
They follow from 
\[\,\mathbf{Q}^{-1}=\left[\begin{smallmatrix}\cdot & -(\mathbf{Q}_{rr}-\mathbf{Q}_{rn}\mathbf{Q}_{nn}^{-1}\mathbf{Q}_{nr})^{-1}\mathbf{Q}_{rn}\mathbf{Q}_{nn}^{-1}\\ -\mathbf{Q}_{nn}^{-1}\mathbf{Q}_{nr}(\cdot) & \cdot\end{smallmatrix}\right]\].

\medskip

\noindent By definition $a_j=\mathbf{v}_j^\top \mathbf{x}_n$ and $\mathbf{y}=\mathbf{H}\mathbf{x}_r+\boldsymbol{\omega}$ with
$\boldsymbol{\omega}\perp\mathbf{x}$. Hence
\begin{equation}\label{eq:Cov-aj-y}
\Cov(a_j,\mathbf{y})
=\mathbf{v}_j^\top \Cov(\mathbf{x}_n,\mathbf{H}\mathbf{x}_r)
=\mathbf{v}_j^\top \mathbf{C}_{nr}\mathbf{H}^\top.
\end{equation}
Using \eqref{eq:block-identities}, this becomes
\[
\Cov(a_j,\mathbf{y})
=-\,\mathbf{v}_j^\top \mathbf{Q}_{nn}^{-1}\mathbf{Q}_{nr}\mathbf{C}_{rr}\,\mathbf{H}^\top.
\]
Therefore the \emph{numerator} of $\rho_j^2$ is
\begin{align}
\mathrm{Num}
&=\Cov(a_j,\mathbf{y})^\top \mathbf{C}_y^{-1}\Cov(a_j,\mathbf{y})\nonumber\\
&=\mathbf{H}\mathbf{C}_{rr}\mathbf{Q}_{rn}\mathbf{Q}_{nn}^{-1}\mathbf{v}_j\;\cdot\;
\mathbf{C}_y^{-1}\;\cdot\;
\mathbf{v}_j^\top \mathbf{Q}_{nn}^{-1}\mathbf{Q}_{nr}\mathbf{C}_{rr}\mathbf{H}^\top.\label{eq:Num-raw}
\end{align}
Introduce $\mathbf{A}:=\mathbf{H}\mathbf{C}_{rr}^{1/2}$ and note $\mathbf{C}_y=\mathbf{A}\mathbf{A}^\top+\sigma^2\mathbf{I}$.
Define also
\[
\mathbf{M}:=\mathbf{C}_{rr}^{1/2}\mathbf{H}^\top\mathbf{C}_y^{-1}\mathbf{H}\mathbf{C}_{rr}^{1/2}
=\mathbf{A}^\top(\mathbf{A}\mathbf{A}^\top+\sigma^2\mathbf{I})^{-1}\mathbf{A}.
\]
Let $\mathbf{w}_j:=\mathbf{C}_{rr}^{1/2}\mathbf{Q}_{rn}\mathbf{Q}_{nn}^{-1}\mathbf{v}_j$.
Then \eqref{eq:Num-raw} rewrites compactly as
\begin{equation}\label{eq:Num-wMw}
\mathrm{Num}=\mathbf{w}_j^\top \mathbf{M}\,\mathbf{w}_j.
\end{equation}

\medskip
\noindent Take an SVD $\mathbf{A}=\mathbf{V}\mathbf{S}\mathbf{V}^\top$, $\mathbf{S}=\diag(s_i)\ge 0$.
Then
\[
\mathbf{M}
=\mathbf{V}\mathbf{S}\,\big(\mathbf{S}^2+\sigma^2\mathbf{I}\big)^{-1}\mathbf{S}\mathbf{V}^\top
=\mathbf{V}\,\diag\!\Big(\frac{s_i^2}{s_i^2+\sigma^2}\Big)\mathbf{V}^\top \;\preceq\; \mathbf{I},
\]
with equality iff $\sigma^2=0$ (and then $\mathbf{M}=\mathbf{I}$). Consequently,
\begin{equation}\label{eq:Num-bound}
\mathrm{Num}=\mathbf{w}_j^\top \mathbf{M}\mathbf{w}_j \;\le\; \mathbf{w}_j^\top \mathbf{w}_j
=\mathbf{v}_j^\top \mathbf{Q}_{nn}^{-1}\mathbf{Q}_{nr}\mathbf{C}_{rr}\mathbf{Q}_{rn}\mathbf{Q}_{nn}^{-1}\mathbf{v}_j.
\end{equation}

\medskip
\noindent Using \eqref{eq:block-identities},
\begin{align}
\mathrm{Den}
&:=\Var(a_j)=\mathbf{v}_j^\top \mathbf{C}_{nn}\mathbf{v}_j\nonumber\\
&=\mathbf{v}_j^\top \mathbf{Q}_{nn}^{-1}\mathbf{v}_j
+\mathbf{v}_j^\top \mathbf{Q}_{nn}^{-1}\mathbf{Q}_{nr}\mathbf{C}_{rr}\mathbf{Q}_{rn}\mathbf{Q}_{nn}^{-1}\mathbf{v}_j.
\label{eq:Den-raw}
\end{align}
Since $\mathbf{Q}_{nn}\mathbf{v}_j=\mu_j\mathbf{v}_j$ and $\|\mathbf{v}_j\|_2=1$, we have
\[
\mathbf{Q}_{nn}^{-1}\mathbf{v}_j=\mu_j^{-1}\mathbf{v}_j,\qquad
\mathbf{v}_j^\top \mathbf{Q}_{nn}^{-1}\mathbf{v}_j=\mu_j^{-1}.
\]
Define the scalar
\[
c_j:=\mathbf{v}_j^\top \big(\mathbf{Q}_{rn}\mathbf{C}_{rr}\mathbf{Q}_{nr}\big)\mathbf{v}_j\;\ge 0,
\]
and observe that
\[
\mathbf{v}_j^\top \mathbf{Q}_{nn}^{-1}\mathbf{Q}_{nr}\mathbf{C}_{rr}\mathbf{Q}_{rn}\mathbf{Q}_{nn}^{-1}\mathbf{v}_j
=\mu_j^{-2}\,c_j.
\]
Plugging into \eqref{eq:Den-raw} gives
\begin{equation}\label{eq:Den-final}
\mathrm{Den}=\mu_j^{-1}+\mu_j^{-2}\,c_j.
\end{equation}

\medskip

\noindent From \eqref{eq:Num-bound} and \eqref{eq:Den-final},
\[
\rho_j^2
=\frac{\mathrm{Num}}{\mathrm{Den}}
\;\le\;
\frac{\mu_j^{-2}c_j}{\mu_j^{-1}+\mu_j^{-2}c_j}
=\frac{c_j}{c_j+\mu_j}.
\]
If $\sigma^2=0$ and $\mathbf{H}=\mathbf{I}$, then $\mathbf{C}_y=\mathbf{C}_{rr}$ and hence
$\mathbf{M}=\mathbf{I}$, so inequality \eqref{eq:Num-bound} is an equality and
$\rho_j^2=\frac{c_j}{c_j+\mu_j}$.

\section{Graph structures}

\subsection{Additional graphs} \label{app:additional_graphs}
We analyze two additional graph topologies.  \vspace{-0.5cm}

\paragraph{Random-walk normalized graph (\(\mathbf{L}_{\mathrm{rw}}\)).}
Starting from an undirected graph with adjacency \(\mathbf{A}\) and degree matrix \(\mathbf{D}\),
The random-walk normalized Laplacian is
\[
\mathbf{L}_{\mathrm{rw}} \;=\; \mathbf{I}-\mathbf{D}^{-1}\mathbf{A}.
\]
It can be read as \(\mathbf{L}_{\mathrm{rw}}=\mathbf{I}-\mathbf{P}\), where
\(\mathbf{P}=\mathbf{D}^{-1}\mathbf{A}\) is the one-step Markov transition; hence
\(\mathbf{L}_{\mathrm{rw}}\mathbf{1}=\mathbf{0}\) and the spectrum reflects how probability mass
mixes locally. This normalization corrects for degree imbalance: edges from low-degree nodes are
not underweighted, so the quadratic form \(\mathbf{x}^\top\mathbf{L}_{\mathrm{rw}}\mathbf{x}\)
penalizes variations relative to local node degree. In our pipeline, we never use
\(\mathbf{L}_{\mathrm{rw}}\) directly; we act on the null component through the null-restricted
operator \(\mathbf{T}=\mathbf{P}_n\mathbf{L}_{\mathrm{rw}}\mathbf{P}_n\) and compute its smooth
eigenmodes to build \(\mathbf{S}\). 

\paragraph{Symmetric normalized graph (\(\mathbf{L}_{\mathrm{sym}}\)).}
The symmetric normalized Laplacian is
\[
\mathbf{L}_{\mathrm{sym}} \;=\; \mathbf{I}-\mathbf{D}^{-1/2}\mathbf{A}\,\mathbf{D}^{-1/2},
\]
which is symmetric positive semidefinite and yields the energy
\(
\mathbf{x}^\top\mathbf{L}_{\mathrm{sym}}\mathbf{x}
=\tfrac{1}{2}\sum_{i,j} w_{ij}\bigl(\tfrac{x_i}{\sqrt{d_i}}-\tfrac{x_j}{\sqrt{d_j}}\bigr)^2.
\)
It equalizes contributions across nodes with different degrees and, on regular grids, is similar
to \(\mathbf{L}_{\mathrm{rw}}\) (hence they share eigenvalues up to a similarity transform).
Within GSNR, we apply the same null restriction \(\mathbf{T}=\mathbf{P}_n\mathbf{L}_{\mathrm{sym}}\mathbf{P}_n\)
and take its \(p\) smallest eigenpairs to span the graph-smooth null subspace used by
\(\mathbf{S}\). 
In practice, both normalizations deliver the intended
effect, concentrating smooth, predictable variation in \(\Null(\mathbf{H})\), with small empirical
differences that reflect how degree reweighting shapes the spectrum of \(\mathbf{T}\).

\begin{figure}
    \centering
    \includegraphics[width=\linewidth]{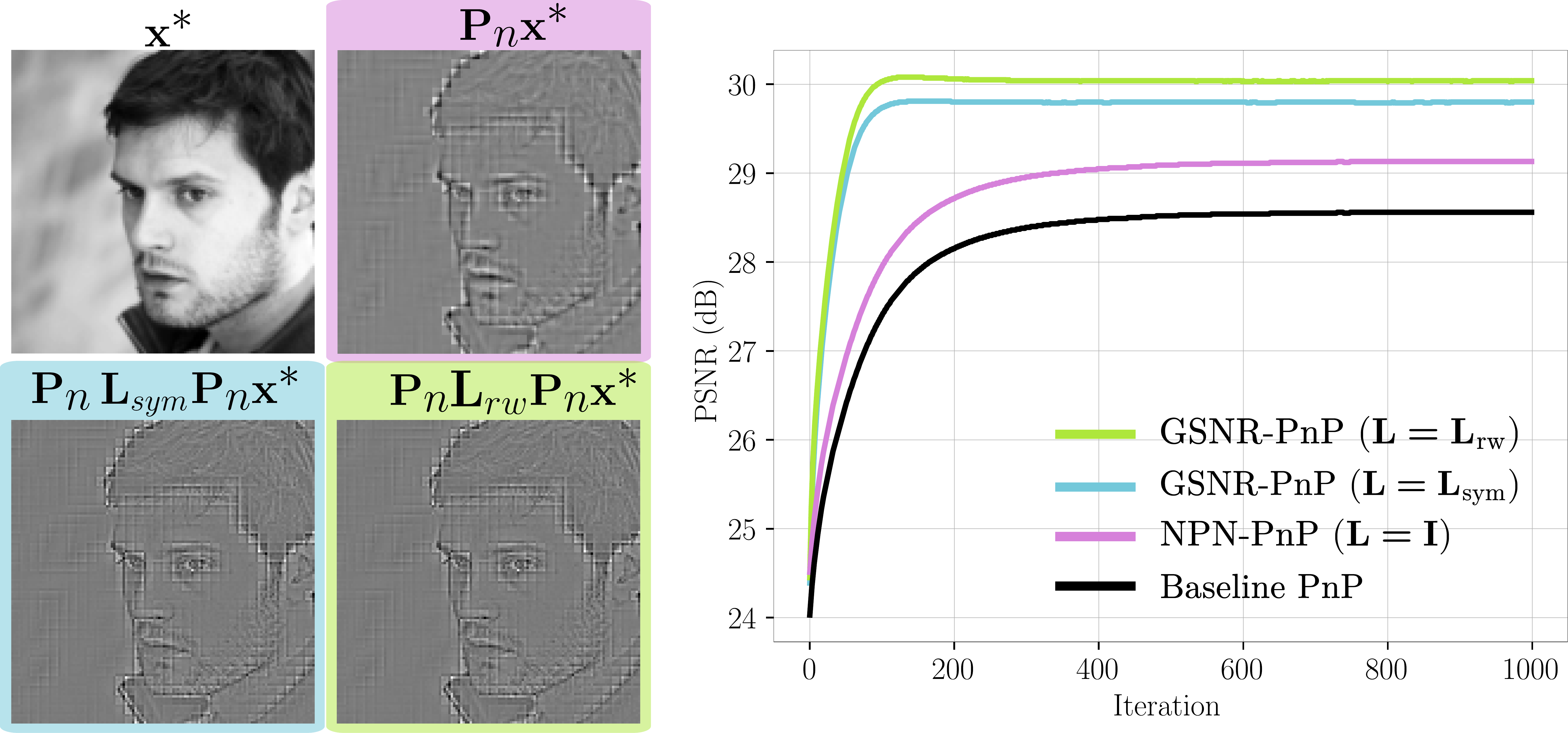}
    \caption{SR with different graph Laplacians. 
Left: ground truth $\mathbf{x}^\ast$, its null component $\mathbf{P}_n\mathbf{x}^\ast$, and
graph-smoothed null responses $\mathbf{P}_n\mathbf{L}_{\mathrm{sym}}\mathbf{P}_n\mathbf{x}^\ast$ 
and $\mathbf{P}_n\mathbf{L}_{\mathrm{rw}}\mathbf{P}_n\mathbf{x}^\ast$. 
Right: PSNR vs.\ iteration for GSNR-PnP with 
$\mathbf{L} \in \{\mathbf{L}_{\mathrm{sym}}, \mathbf{L}_{\mathrm{rw}}, \mathbf{I}\}$, and the PnP baseline.}
    \label{fig:rw_sym}
\end{figure}

\begin{figure*}[!t]
    \centering
    \includegraphics[width=\linewidth]{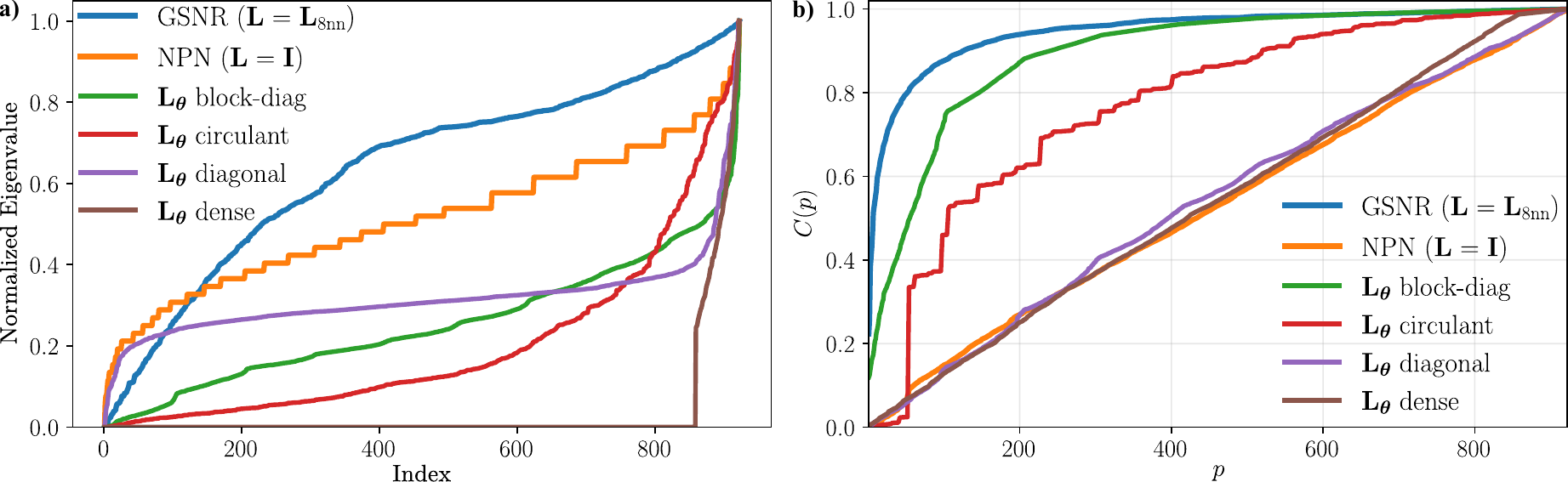}
\vspace{-0.7cm}
    \caption{\textbf{(a)} Variation of $\mathbf{T}$ normalized eigenvalues for $\mathbf{L}_{\theta}$ with respect to their index in CS. \textbf{(b)} Coverage of $\mathbf{S}$ with $\mathbf{L}_{\theta}$ in CS.}
    \label{fig:L_theta}
\vspace{-0.5cm}
\end{figure*}

In Figure \ref{fig:rw_sym}, the left panels visualize the invisible component and its graph–smoothed variants:
$\mathbf{P}_n\mathbf{x}^{\ast}$ (bottom–left) highlights edges and fine textures that lie in the
null space of the sensing operator, while applying the null–restricted Laplacians
$\mathbf{P}_n\mathbf{L}_{\mathrm{sym}}\mathbf{P}_n$ and
$\mathbf{P}_n\mathbf{L}_{\mathrm{rw}}\mathbf{P}_n$ further emphasizes coherent edge geometry and
suppresses isolated, high–frequency speckles. The two normalizations yield very similar
structures, with $\mathbf{L}_{\mathrm{rw}}$ marginally enhancing sharp contours, consistent with its degree-normalized reweighting. This visual evidence aligns with the goal of GSNR: to impose
structure \emph{only} in the null space and make its content smoother and more predictable.

The PSNR–iteration curves (right) show that graph–limited null–space designs
(GSNR–PnP with $\mathbf{L}_{\mathrm{sym}}$ or $\mathbf{L}_{\mathrm{rw}}$) both \emph{converge
faster} and \emph{stabilize at higher PSNR} than the geometry–free alternative
($\mathbf{L}=\mathbf{I}$) and the PnP baseline. Among the graph choices, the
$\mathbf{L}_{\mathrm{rw}}$ variant exhibits the steepest initial rise and the highest plateau,
while $\mathbf{L}_{\mathrm{sym}}$ tracks closely behind; both consistently dominate
$\mathbf{L}=\mathbf{I}$. This behavior matches the theory: normalized Laplacians produce a more informative spectrum for the null–restricted operator $\mathbf{T}=\mathbf{P}_n\mathbf{L}
\mathbf{P}_n$, yielding higher spectral coverage and stronger statistical coupling, which in turn
improves both the fixed point (final PSNR) and the transient (speed of convergence).

\subsection{Learning a Structured Laplacian} \label{app:graph_as_a_matrix}

The learning objective maximizes the average coverage over a prescribed set \(\mathcal{P}\) (typically 10 equispaced values):
\[
\max_{\mathbf{L}_{\theta} \succeq 0} \;\; 
\frac{1}{|\mathcal{P}|}\sum_{p\in\mathcal{P}} C(p),
\]

\noindent using $\mathbf{T}=\mathbf{P}_n \mathbf{L}_{\theta} \mathbf{P}_n$ for the creation of $\mathbf{S}$ as in Sec. \ref{subsect:graphlimited}. Each parametrization enforces \(\mathbf{L}_{\theta}\succeq 0\) via a symmetric construction plus a small \(\varepsilon \mathbf{I}\) to keep eigenvalues nonnegative.

\paragraph{Dense (low-rank PSD).}
Learn \(\boldsymbol{\Theta}\in\mathbb{R}^{n\times r}\):
\[
\mathbf{L}_{\theta} = \boldsymbol{\Theta}\,\boldsymbol{\Theta}^\top + \varepsilon\,\mathbf{I}.
\]

\paragraph{Diagonal.}
Let \(\boldsymbol{\theta}\in\mathbb{R}^n\) be learnable logits and define \(\mathrm{softplus}(t) \triangleq \log\!\bigl(1+ e^{t}\bigr)\). Set
\[
\mathbf{d} = \mathrm{softplus}(\boldsymbol{\theta}) + \varepsilon,\qquad
\mathbf{L}_{\theta} = \mathrm{diag}(\mathbf{d}).
\]

\paragraph{Circulant (wrap-around convolution).}
Learn a kernel \(\boldsymbol{\Theta}\in\mathbb{R}^{E\times E}\); let \(\mathbf{C}_{\mathrm{circ}}(\boldsymbol{\Theta})\in\mathbb{R}^{n\times n}\) denote the circular (wrap-around) convolution operator on vectorized images. Then
\[
\mathbf{L}_{\theta} = \mathbf{C}_{\mathrm{circ}}(\boldsymbol{\Theta})^\top\,\mathbf{C}_{\mathrm{circ}}(\boldsymbol{\Theta}) + \varepsilon\,\mathbf{I}.
\]
\paragraph{Block-diagonal.}
Partition \(n\) into \(b=n/B\) blocks of size \(B\). Learn \(\boldsymbol{\Theta}_i \in \mathbb{R}^{B\times B}\) for \(i=1,\dots,b\). With
\[
\mathrm{blockdiag}(\mathbf{A}_1,\dots,\mathbf{A}_b) \triangleq
\begin{bmatrix}
\mathbf{A}_1 & & \\
& \ddots & \\
& & \mathbf{A}_b
\end{bmatrix},
\]
we set
\[
\mathbf{L}_{\theta} = \mathrm{blockdiag}\bigl(
\boldsymbol{\Theta}_1 \boldsymbol{\Theta}_1^\top + \varepsilon \mathbf{I}_B,\;
\dots,\;
\boldsymbol{\Theta}_b \boldsymbol{\Theta}_b^\top + \varepsilon \mathbf{I}_B
\bigr).
\]

Fig. \ref{fig:L_theta} shows the comparison of a) normalized $\mathbf{T}$-eigenvalues and b) $\mathbf{S}$-coverage for the different parametrizations of $\mathbf{L}_{\theta}$. Using a Laplacian graph such as $\mathbf{L}_{8nn}$ in GSNR yields the best spectral performance and coverage for any value of $p$. Although NPN has an adequate spectrum, its coverage is very limited, whereas the opposite is true for block-diag.

\section{Settings for construction null-restricted Laplacian in practice}\label{app:t_constructio}

We never materialize \(\mathbf{P}_n\) or \(\mathbf{T}\) as dense matrices. Instead, \(\mathbf{P}_n\mathbf{v}=\mathbf{v}-\mathbf{H}^\top(\mathbf{H}\mathbf{H}^\top)^{-1}\mathbf{H}\mathbf{v}\) is exposed as a callable projector using a factorization of \(\mathbf{H}\mathbf{H}^\top\), and \(\mathbf{T}\) is wrapped as a SciPy \texttt{LinearOperator} that applies \(\mathbf{T}\mathbf{x}=\mathbf{P}_n\bigl(\mathbf{L}\,\mathbf{P}_n\mathbf{x}\bigr)\) on the fly. We then invoke ARPACK \cite{lehoucq1998arpack} via \texttt{eigsh} with spectral method to extract the \(k\) smallest–magnitude eigenpairs of \(\mathbf{T}\) (the smoothest graph–null modes), respecting the constraint \(k\leq n\) and defaulting to \(k=\min(q,n-1)\) when unspecified. The routine returns eigenvectors as columns \(\mathbf{U}\in\mathbb{R}^{n\times k}\) and eigenvalues \(\{\mu_j\}\); we set \(\mathbf{S}_{\text{full}}=\mathbf{U}^\top\) so that the rows of \(\mathbf{S}\) form an orthonormal basis of the selected null–space subspace, and finally truncate to the first \(p\) rows for training/inference.
\textcolor{black}{The process is performed for different graph Laplacians $\mathbf{L}$, and the structure selection is guided by our spectral criteria, i.e., maximum coverage/predictability over the first \(p\) modes.} The \emph{demosaicing} case follows the same flow but loads a precomputed sparse \(\mathbf{H}\), optionally lifts the Laplacian to multi–channel form with a Kronecker product \(\mathbf{L}\leftarrow \mathbf{I}_C\otimes\mathbf{L}\).

For the numerical results of this work, we empirically set the dimension $p$. However, our framework provides a principled evaluation on how to select $p$ based on the null-space coverage. Let \(\boldsymbol{\lambda}=(\lambda_1,\dots,\lambda_q)\) be the eigenvalues of
\(\mathrm{Cov}(\mathbf{x}_n)\) in the graph-smooth basis, ordered so that
\(\lambda_1\ge\cdots\ge\lambda_q>0\). Define the cumulative coverage
\(C(p)=\big(\sum_{i=1}^{p}\lambda_i\big)/\big(\sum_{i=1}^{q}\lambda_i\big)\).
We select the effective dimension \(p^\star\) as the smallest \(p\) that simultaneously
achieves a target coverage level and lies on a plateau of the coverage curve, following Algorithm \ref{alg:coverage_p}.

\begin{algorithm}[h]
\caption{Coverage-based automatic selection of \(p\)}
\label{alg:coverage_p}
\begin{algorithmic}[1]
\Require Eigenvalues \(\lambda_1\ge\cdots\ge\lambda_q>0\);
         coverage target \(\kappa\in(0,1)\) (e.g.\ \(\kappa=0.95\));
         slope tolerance \(\delta>0\) (e.g.\ \(\delta=10^{-3}\));
         plateau length \(L\in\mathbb{N}\) (e.g.\ \(L=10\)).
\State Compute total variance \(S \leftarrow \sum_{i=1}^{q}\lambda_i\).
\State For \(p=1,\dots,q\), compute coverage
       \(C(p) \leftarrow \big(\sum_{i=1}^{p}\lambda_i\big)/S\).
\State For \(p=1,\dots,q\), compute incremental gains
       \(\Delta C(p) \leftarrow C(p)-C(p-1)\), with \(C(0)\equiv 0\).
\For{\(p = 1,\dots,q\)}
    \If{\(C(p) \ge \kappa\)}
        \State Check plateau condition:
               \(\max\{\Delta C(p),\dots,\Delta C(\min(p+L-1,q))\} \le \delta\).
        \If{plateau condition holds}
            \State \textbf{return} \(p^\star \leftarrow p\).
        \EndIf
    \EndIf
\EndFor
\State If no \(p\) satisfies the above, set \(p^\star \leftarrow q\) (use all modes).
\end{algorithmic}
\end{algorithm}

\section{Ablation Studies of the Graph Regularizer}
\label{app:graph_reg}
\textcolor{black}{Recall that GSNR is \emph{solver-agnostic}: it augments the data-fidelity objective with the terms
\(\gamma\|\mathrm{G}^{\ast}(\mathbf{y})-\mathbf{S}\mathbf{x}\|_{2}^{2}+\tfrac{\gamma_{g}}{2}\mathbf{x}^{\top}\mathbf{T}\mathbf{x}\).
Consequently, GSNR can be incorporated into any iterative solver, e.g., ADMM, HQS, FISTA, PD, or diffusion-based methods (see \ref{app:GSNR_in_DM}), by including these terms in its \(\mathbf{x}\)-update.}
Before showing further experiments, we describe the baseline algorithms and the GSNR versions. 

\textbf{Plug-and-Play.}
In this case, we used the PGD-PnP method. 
Algorithm \ref{alg:gsnr_pgd} shows the GSNR PnP-PGD modification.

\begin{algorithm}[h]
\footnotesize 
\caption{GSNR PnP-PGD with \textcolor{OliveGreen!100}{null-space} and \textcolor{blue!100}{graph regularization}}
\label{alg:gsnr_pgd}
\begin{algorithmic}[1]
\Require $K,\mathbf{H},\mathbf{y},\alpha,\eta$, $\gamma,\gamma_g,\mathrm{G}^*,\mathbf{S}$
  \State $\mathbf{x}_0=\mathbf{H}^\top \mathbf{y} +\mathbf{S}^\top \mathrm{G}^*(\mathbf{y} ) $
  \For{$i=1,\dots,K$}
 
    \State \begin{align*}
    \mathbf{x}_i\gets\mathbf{x}_{i-1}-\alpha&\left(\mathbf{H}^\top(\mathbf{Hx}_{i-1} -\mathbf{y})+\textcolor{OliveGreen!100}{\gamma \mathbf{S}^\top(\mathbf{Sx}_{i-1} -\mathrm{G}^*(\mathbf{y}))}\right.\\&\left.+ \textcolor{blue!100}{ \gamma_g\mathbf{Tx}_{i-1}}\right)\end{align*}
    \State $\mathbf{x}_i\gets\mathrm{D}_\eta(\mathbf{x}_i)$
  \EndFor
  \State\Return $\mathbf{x}_i$
\end{algorithmic}
\end{algorithm}

\begin{figure}[!t]
\centering
\includegraphics[width=\linewidth]{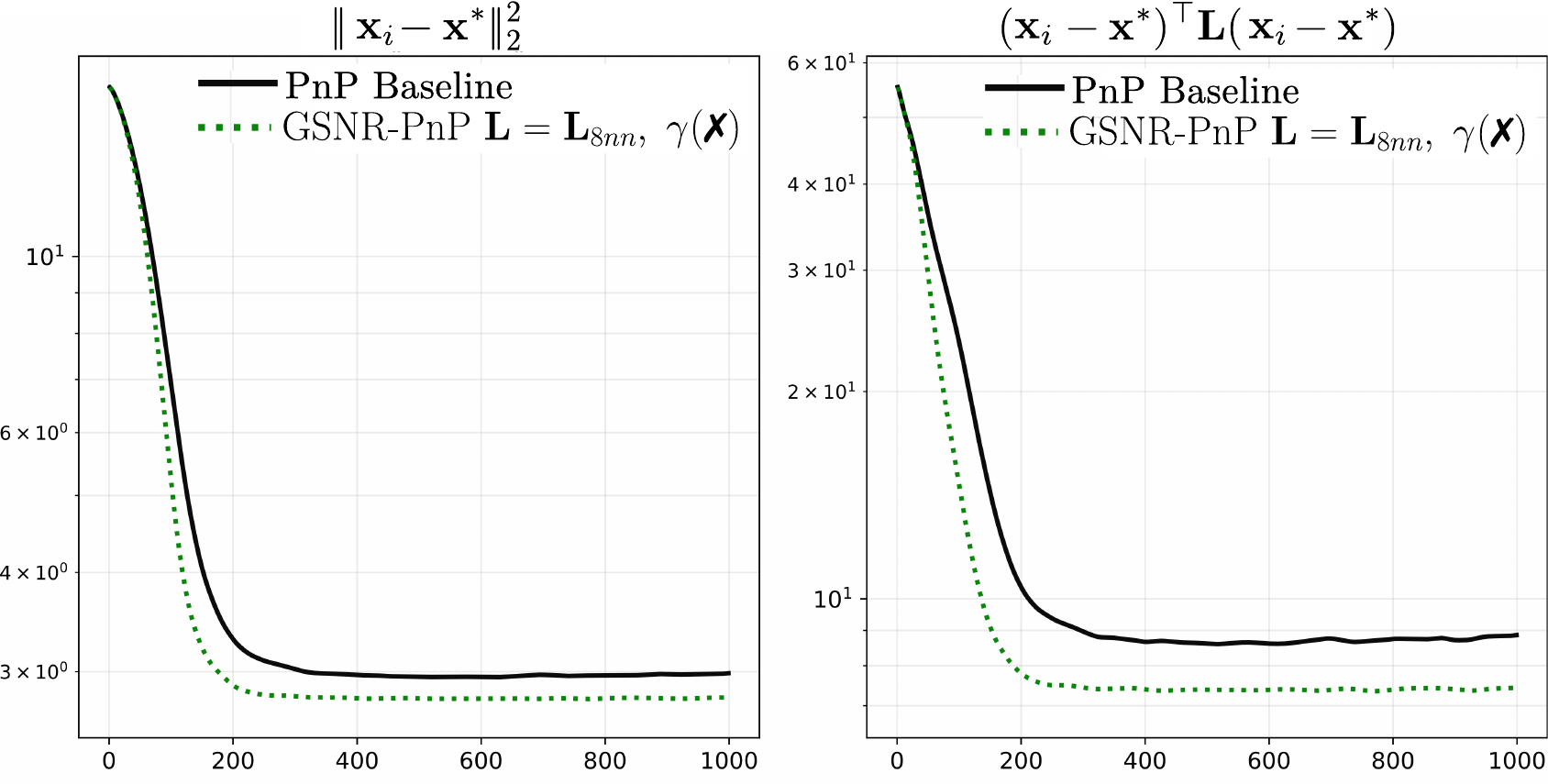}
        \vspace{-0.7cm}
\caption{PnP variant with null-only projector regularizer: illustrative convergence.}
\vspace{-0.4cm}
\label{fig:convergence_with_projector}
\end{figure}

\subsection{Why graph smooth null-space?}  An initial test that shows us the usefulness of the null-space Laplacian is to start from the assumption that the reconstruction error with respect to the ground-truth, \(\mathbf{x}_n^{(i)}\!=\!\mathbf{x}_i-\mathbf{x}^\ast\) is the null-space. Fig. \ref{fig:convergence_with_projector} shows the convergence of the reconstruction and how the null-laplacian regularizer behaves. The baseline (black) shows that improving the quality of our reconstruction also reduces the null-laplacian error \((\mathbf{x}_n^{(i)})^\top \mathbf{L}(\mathbf{x}_n^{(i)})\), even though this term is not considered in the PnP cost function. Adding this term (green) would further improve the reconstruction. This demonstrates the usefulness of including a regularizer that promotes a graph-limited null-space.

\subsection{Why low-dimensional null-space projections?} 
Fig. \ref{fig:pablation} shows the convergence of the reconstruction when using different values of $p$. It can be seen that as the value of $p$ increases, the reconstruction error decreases; however, this reaches a limit, since from $p/n=0.1$ onward, the gain decreases. These results justify the need to use low-dimensional null-space projections, as they enable higher-quality reconstruction without being significantly affected by the decrease in predictability.

\begin{figure}[!t]
    \centering
\includegraphics[width=0.8\linewidth]{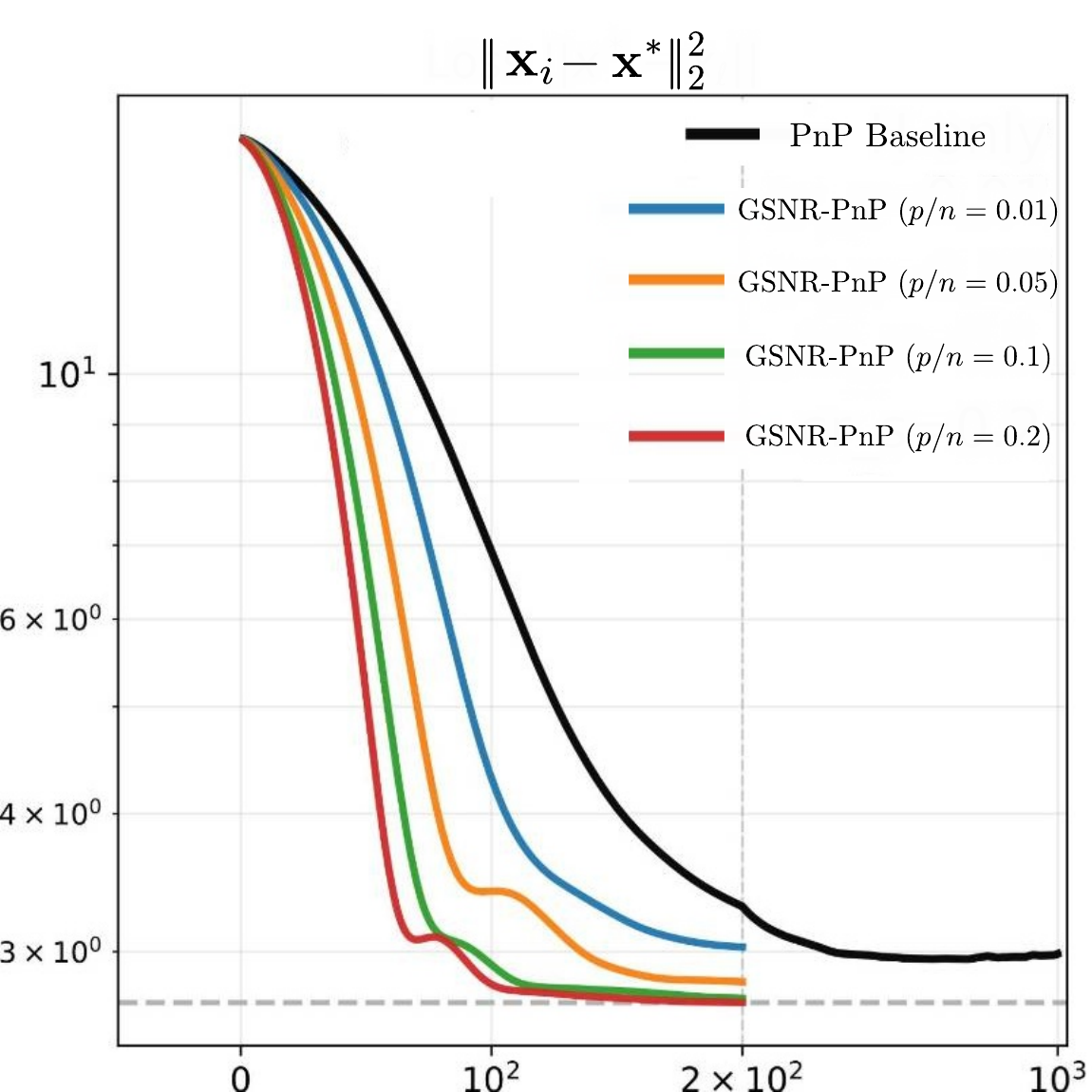}
    \vspace{-0.3cm}
\caption{Low-dimensional null-space.}
    \vspace{-0.4cm}
\label{fig:pablation}
\end{figure}

\subsection{Cost-function ablation}

\begin{figure}[!t]
    \centering
    \includegraphics[width=\linewidth]{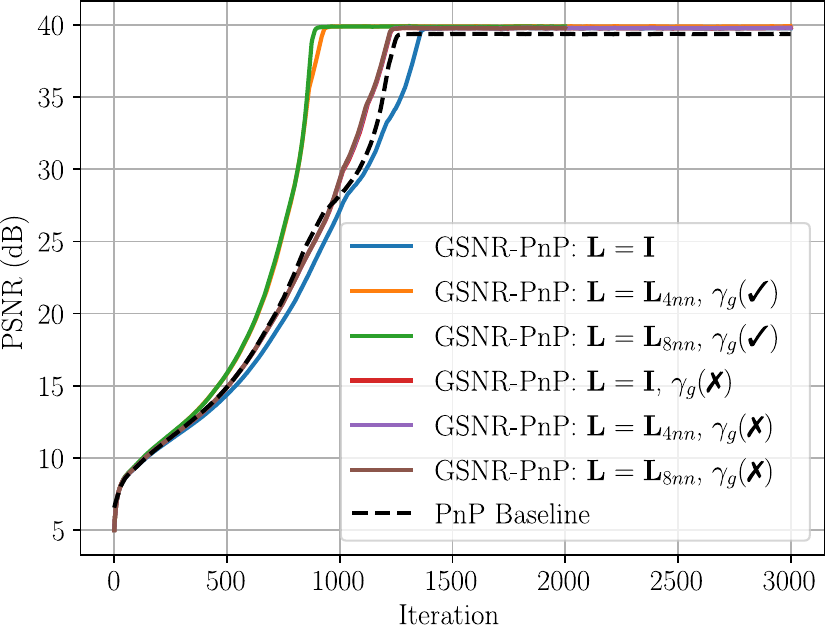}
        \vspace{-0.7cm}
    \caption{Effect of the null-only graph regularizer on GSNR-PnP convergence for demosaicing.
    We plot PSNR versus iteration for GSNR-PnP with different Laplacians
    ($\mathbf{L}=\mathbf{I}$, $\mathbf{L}_{4\mathrm{nn}}$, $\mathbf{L}_{8\mathrm{nn}}$) and with
    the graph-regularization weight $\gamma_{\mathrm{g}}$ either enabled or disabled,
    along with the standard PnP baseline (dashed).}
    \vspace{-0.1cm}
    \label{fig:gsnr_pnp_gamma}
\end{figure}

\begin{figure}[!t]
    \centering
    \includegraphics[width=\linewidth]{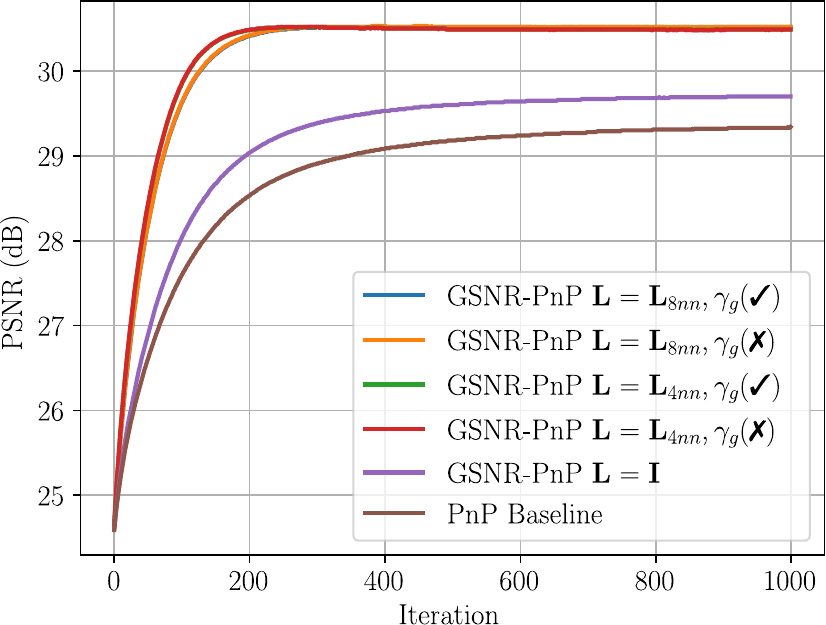}
        \vspace{-0.7cm}
    \caption{Effect of the null-only graph regularizer on GSNR-PnP convergence for super-resolution.
    We plot PSNR versus iteration for GSNR-PnP with different Laplacians
    ($\mathbf{L}=\mathbf{I}$, $\mathbf{L}_{4\mathrm{nn}}$, $\mathbf{L}_{8\mathrm{nn}}$) and with
    the graph-regularization weight $\gamma_{\mathrm{g}}$ either enabled or disabled,
    along with the standard PnP baseline (dashed).}
    \vspace{-0.1cm}
    \label{fig:gsnr_pnp_gamma_sr}
\end{figure}

Figure~\ref{fig:gsnr_pnp_gamma} highlights two effects of the GSNR design: improved conditioning
through the null-only graph regularizer and sensitivity to the choice of Laplacian in demosaicing problem. All GSNR-PnP
variants eventually reach a similar high PSNR plateau, slightly above the PnP baseline, showing
that incorporating the graph-smooth null-space prior does not harm the final reconstruction
quality and can modestly improve it. However, the convergence speed differs significantly:
when $\gamma_{\mathrm{g}}$ is active, GSNR-PnP with
$\mathbf{L}_{4\mathrm{nn}}$ or $\mathbf{L}_{8\mathrm{nn}}$ reaches its peak PSNR in far fewer
iterations than both the geometry-free $\mathbf{L}=\mathbf{I}$ case and the baseline.

When the graph regularizer is turned off, the graph-based methods still
outperform the baseline but converge more slowly, with trajectories that are closer to the standard PnP.
This confirms the theoretical prediction that the null-only graph term improves the spectrum of
the normal matrix, effectively “lifting’’ the null directions, while the GSNR basis itself governs
the final achievable PSNR. In practice, combining a graph Laplacian with a nonzero
$\gamma_{\mathrm{g}}$ yields the best trade-off: fast convergence to high-quality solutions with
minimal overhead in the PnP update. Similar analysis and results are shown in Fig. \ref{fig:gsnr_pnp_gamma_sr}, for the image super-resolution problem, where the graph regularizer slightly increases convergence speed.

\begin{figure}[!t]
    \centering
    \includegraphics[width=0.8\linewidth]{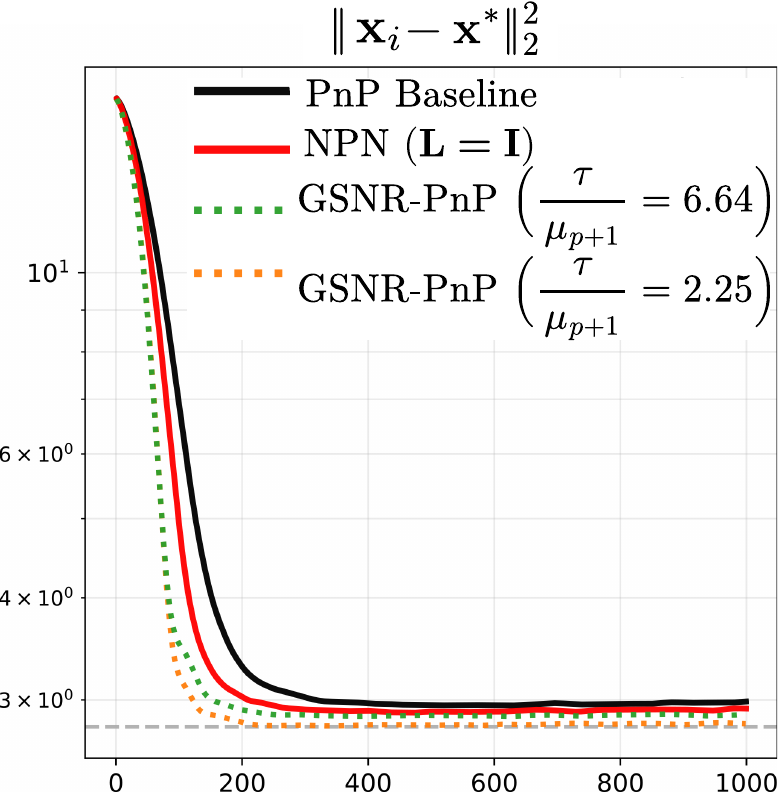}
    \vspace{-0.3cm}
\caption{Bound vs.\ quality. }
    \vspace{-0.3cm}
\label{fig:boundquality}
\end{figure}

 \subsection{Minimax optimality bound}
\textcolor{black}{To experimentally validate the theory of Theorem \ref{th:Th1}, two different operators $\mathbf{S}$ were tested with $p=0.1n$ and $\tau=1$ for Fig. \ref{fig:boundquality}. In the green and yellow cases, there are two extreme values of the bound $\frac{\tau}{\mu_{p+1}}$ that demonstrate the theorem's postulate, since a $\mathbf{S}$ with a lower bound (2.25) results in better reconstruction.}

 \subsection{Fixed-point convergence}
\label{app:fixed_point_convergence}
 When analyzing an extreme case in Fig. \ref{fig:fixedpoint}, it can be observed that the proposed GSNR method (dotted lines) converges to a fixed point without diverging as iterations progress, achieving even faster convergence than NPN and the baseline.

\begin{figure}
    \centering
\includegraphics[width=0.8\linewidth]{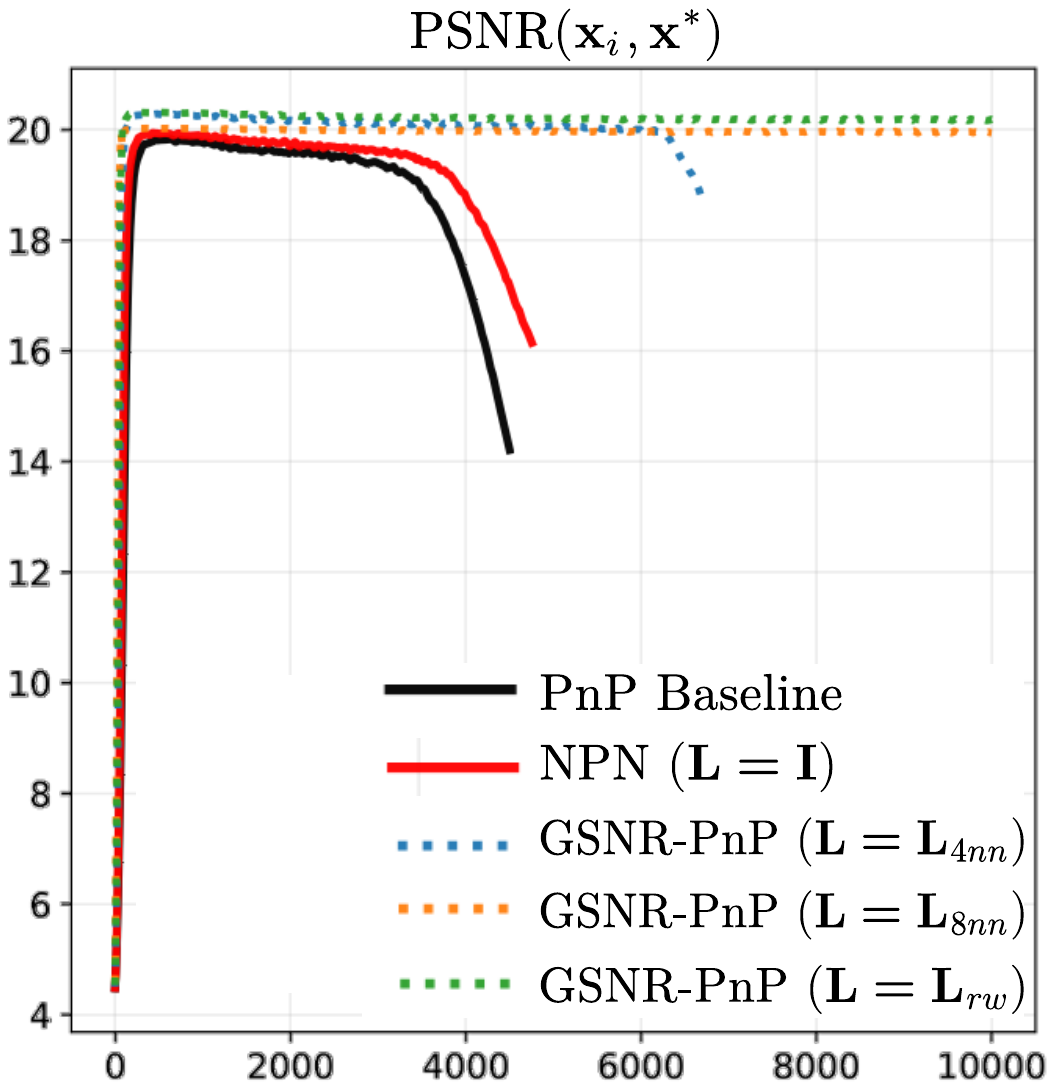}
        \vspace{-0.3cm}
\caption{Fixed-point convergence.}\label{fig:fixedpoint}
        \vspace{-0.1cm}
\end{figure}

\section{Coverage curve}
\label{app:coverage_in_other_tasks}

Figure~\ref{fig:coverage_sr} shows that graph-based Laplacians concentrate null-space variance
into a small number of modes. For $\mathbf{L}_{4\mathrm{nn}}$, $\mathbf{L}_{\mathrm{rw}}$, and
$\mathbf{L}_{\mathrm{sym}}$, the coverage rises steeply and reaches almost full variance with
only a fraction of the null-space dimension, whereas the identity Laplacian
$\mathbf{L}=\mathbf{I}$ exhibits an almost perfectly linear curve $C(p)\approx p/q$, meaning
that coverage grows only proportionally to the dimension, and no “early” compression occurs. The
$\mathbf{L}_{8\mathrm{nn}}$ graph still offers a substantial advantage over $\mathbf{L}=\mathbf{I}$.
But its curve saturates below the others, indicating slightly less concentrated variance. Overall,
These results confirm the theoretical prediction that graph-smooth null modes provide much better
coverage than geometry-free bases: a relatively small $p$ already captures most null-space energy
for the normalized and grid Laplacians, while the identity requires many more modes to achieve

\begin{figure}[t!]
\centering
\includegraphics[width=\linewidth]{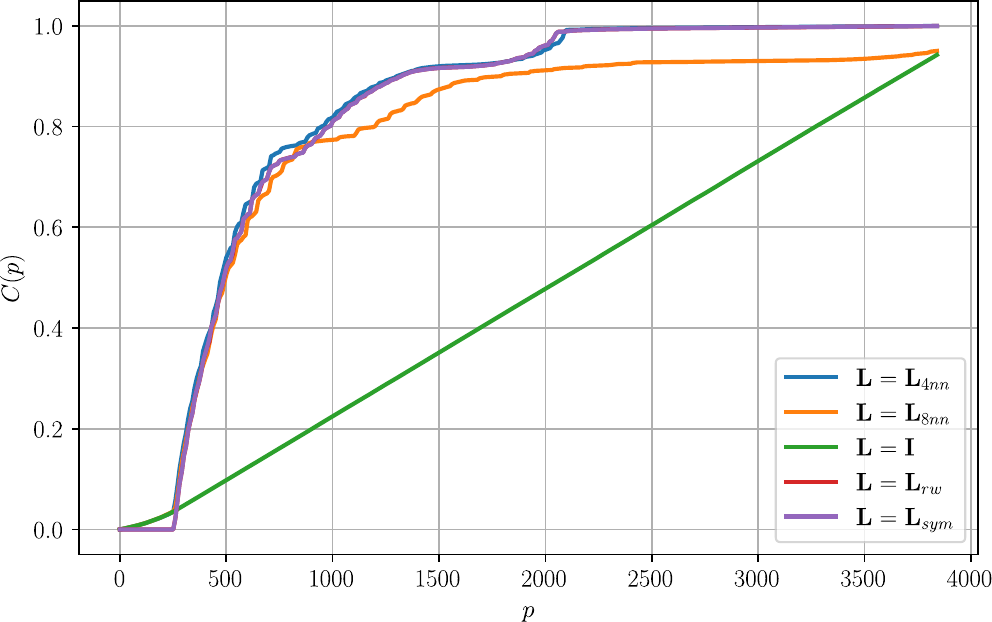}
\vspace{-0.7cm}
\caption{Spectral coverage curves $C(p)$ for super-resolution, comparing different
    Laplacian choices in the GSNR construction: grid $\mathbf{L}_{4\mathrm{nn}}$,
    grid $\mathbf{L}_{8\mathrm{nn}}$, random-walk normalized $\mathbf{L}_{\mathrm{rw}}$,
    symmetric normalized $\mathbf{L}_{\mathrm{sym}}$, and the geometry-free baseline
    $\mathbf{L}=\mathbf{I}$. Coverage $C(p)$ is the fraction of null-space variance captured by the first $p$ graph-smooth modes.}
\label{fig:coverage_sr}
\vspace{-0.4cm}
\end{figure}

\section{{GSNR inclusion in Diffusion-based solvers}}
\label{app:GSNR_in_DM}
\subsection{DPS \cite{dps}.}
We denote $K$ is the number of reverse diffusion steps, and $i\in\{0,\dots,K-1\}$ is the
reverse-time index; $\mathbf{x}_i\in\mathbb{R}^n$ is the current latent state and
$\mathbf{x}_K\sim\mathcal{N}(\mathbf{0},\mathbf{I})$ is the Gaussian start;
$\hat{\mathbf{s}}=\mathbf{s}_\theta(\mathbf{x}_i,i)$ is the score/noise estimate produced by the
network with parameters $\theta$; $\hat{\mathbf{x}}_0$ is the network’s prediction of the clean
sample at step $i$; $\alpha_i\in(0,1]$ is the per-step retention factor,
$\beta_i=1-\alpha_i$ is the noise increment, and
$\bar{\alpha}_i=\prod_{j=1}^{i}\alpha_j$ is the cumulative product
(with $\bar{\alpha}_0=1$); $\zeta_i>0$ is the data-consistency step size and
$\tilde{\sigma}_i\ge 0$ is the sampling noise scale at step $i$;
$\mathbf{z}\sim\mathcal{N}(\mathbf{0},\mathbf{I})$ is i.i.d.\ Gaussian noise;
$\mathbf{x}'_{i-1}$ denotes the pre–data-consistency iterate before applying the gradient
correction. In the GSNR version, we further introduce a learned null-space predictor
$\mathrm{G}(\mathbf{y})\approx\mathbf{S}\mathbf{x}^\ast$ and a weight $\gamma>0$ for the
graph-smooth null-space penalty $\|\mathrm{G}(\mathbf{y})-\mathbf{S}\hat{\mathbf{x}}_0\|_2^2$. Algorithm \ref{alg:gsnr_dps} shows the integration of GSNR in DPS.

\begin{algorithm}[!t]
\footnotesize
\caption{GSNR–DPS Sampling with \textcolor{OliveGreen!100}{null-space} and \textcolor{blue!100}{graph regularization}}
\label{alg:gsnr_dps}
\begin{algorithmic}[1]
\Require $K,\;\mathbf{H},\;\mathbf{y},\;\{\,\zeta_i\,\}_{i=1}^K,\;\{\,\tilde{\sigma}_i\,\}_{i=1}^K,\;
         \gamma,\;\gamma_{\mathrm{g}},\;\mathbf{S},\;\mathrm{G}^\ast,\;\mathbf{P}_n,\;\mathbf{L}$
\State $\mathbf{x}_K \sim \mathcal{N}(\mathbf{0},\mathbf{I})$
\For{$i = K\!-\!1,\dots,0$}
  \State $\hat{\mathbf{s}} \gets \mathbf{s}_\theta(\mathbf{x}_i,\, i)$
  \State $\hat{\mathbf{x}}_0 \gets \frac{1}{\sqrt{\bar{\alpha}_i}}\!\left(\mathbf{x}_i + (1-\bar{\alpha}_i)\hat{\mathbf{s}}\right)$
  \State $\mathbf{z} \sim \mathcal{N}(\mathbf{0},\mathbf{I})$
  \State $\mathbf{x}'_{i-1} \gets
      \frac{\sqrt{\alpha_i} (1-\bar{\alpha}_{i-1})}{1-\bar{\alpha}_i}\;\mathbf{x}_i
      + \frac{\sqrt{\bar{\alpha}_{i-1}}\beta_i}{1-\bar{\alpha}_i}\;\hat{\mathbf{x}}_0
      + \tilde{\sigma}_i\,\mathbf{z}$
  \State $\mathbf{x}_{i-1} \gets \mathbf{x}'_{i-1}
      \;-\;\zeta_i\,\nabla_{\mathbf{x}_i}\Bigl[
           \bigl\|\mathbf{y}-\mathbf{H}\,\hat{\mathbf{x}}_0\bigr\|_2^2
           + \textcolor{OliveGreen!100}{\gamma\,\bigl\|\mathrm{G}^\ast(\mathbf{y})-\mathbf{S}\hat{\mathbf{x}}_0\bigr\|_2^2}
           + \textcolor{blue!100}{\gamma_{\mathrm{g}}\,\bigl\|\mathbf{P}_n \mathbf{L}\hat{\mathbf{x}}_0\bigr\|_2^2}
        \Bigr]$
\EndFor
\State \Return $\hat{\mathbf{x}}_0$
\end{algorithmic}
\end{algorithm}

\medskip

\subsection{DiffPIR \cite{DiffPIR}.}
$\sigma_n>0$ denotes the standard deviation of the measurement noise, and $\eta>0$ is the
data–proximal penalty that trades off data fidelity and the denoiser prior inside the subproblem;
$\rho_i \triangleq \eta\,\sigma_n^{2}/\tilde{\sigma}_i^{\,2}$ is the iteration-dependent
weight used in the proximal objective at step $i$;
$\tilde{\mathbf{x}}^{(i)}_{0}$ is the score-model denoised prediction of the clean sample at
step $i$ (before enforcing data consistency);
$\hat{\mathbf{x}}^{(i)}_{0}$ is the solution of the data-proximal subproblem at step $i$;
$\hat{\boldsymbol{\epsilon}}=\bigl(1-\alpha_i\bigr)^{-1/2}
\bigl(\mathbf{x}_i-\sqrt{\bar{\alpha}_i}\,\hat{\mathbf{x}}^{(i)}_0\bigr)$
is the effective noise estimate implied by $(\mathbf{x}_i,\hat{\mathbf{x}}^{(i)}_0)$;
$\boldsymbol{\epsilon}_i\sim\mathcal{N}(\mathbf{0},\mathbf{I})$ is the fresh Gaussian noise
injected at step $i$; and $\zeta\in[0,1]$ mixes deterministic and stochastic updates
($\zeta=0$ fully deterministic, $\zeta=1$ fully stochastic). In the GSNR variant, we again use
a null-space predictor $\mathrm{G}(\mathbf{y})$ and a weight $\gamma>0$ to bias the proximal
subproblem towards graph-smooth null-space coefficients. In Algorithm \ref{alg:gsnr_diffpir}, we show the GSNR modification of DiffPIR.

\begin{algorithm}[!t]
\footnotesize
\caption{GSNR–DiffPIR Sampling \textcolor{OliveGreen!100}{null-space} and \textcolor{blue!100}{graph regularization}}
\label{alg:gsnr_diffpir}
\begin{algorithmic}[1]
\Require $K,\;\mathbf{H},\;\mathbf{y},\;\sigma_n,\;\{\tilde{\sigma}_i\}_{i=1}^K,\;\zeta,\;\eta,\;\gamma,\;
         \gamma_{\mathrm{g}},\;\mathbf{S},\;\mathrm{G}^\ast,\;\mathbf{P}_n,\;\mathbf{L}$
\State Precompute $\rho_i \gets \eta\,\sigma_n^2/\tilde{\sigma}_i^{\,2}$ for $i=1,\dots,K$
\State $\mathbf{x}_K \sim \mathcal{N}(\mathbf{0},\mathbf{I})$
\For{$i=K,\dots,1$}
  \State $\hat{\mathbf{s}} \gets \mathbf{s}_\theta(\mathbf{x}_i,i)$
  \State $\tilde{\mathbf{x}}^{(i)}_0 \gets \frac{1}{\sqrt{\bar{\alpha}_i}}\!\left(\mathbf{x}_i + (1-\bar{\alpha}_i)\hat{\mathbf{s}}\right)$
  \State $\hat{\mathbf{x}}^{(i)}_0 \gets \arg\min_{\mathbf{x}}\;
          \bigl\|\mathbf{y}-\mathbf{H}\mathbf{x}\bigr\|_2^2
          + \rho_i\bigl\|\mathbf{x}-\tilde{\mathbf{x}}^{(i)}_0\bigr\|_2^2
          + \textcolor{OliveGreen!100}{\gamma\,\bigl\|\mathrm{G}^\ast(\mathbf{y})-\mathbf{S}\mathbf{x}\bigr\|_2^2}
          + \textcolor{blue!100}{\gamma_{\mathrm{g}}\,\bigl\|\mathbf{P}_n \mathbf{L}\mathbf{x}\bigr\|_2^2}$
  \State $\hat{\boldsymbol{\epsilon}} \gets \frac{1}{\sqrt{1-\alpha_i}}\!\left(\mathbf{x}_i - \sqrt{\bar{\alpha}_i}\,\hat{\mathbf{x}}^{(i)}_0\right)$
  \State $\boldsymbol{\epsilon}_i \sim \mathcal{N}(\mathbf{0},\mathbf{I})$
  \State $\mathbf{x}_{i-1} \gets \sqrt{\bar{\alpha}_{i-1}}\,\hat{\mathbf{x}}^{(i)}_0
    + \sqrt{1-\bar{\alpha}_{i-1}}\big(\sqrt{1-\zeta}\,\hat{\boldsymbol{\epsilon}}+\sqrt{\zeta}\,\boldsymbol{\epsilon}_i\big)$
\EndFor
\State \Return $\hat{\mathbf{x}}^{(1)}_0$
\end{algorithmic}
\end{algorithm}

In addition to the DPS / DiffPIR variables, GSNR uses:
$\mathbf{S}\in\mathbb{R}^{p\times n}$, the graph-smooth null-space projector;
$\mathrm{G}^\ast(\mathbf{y})\approx\mathbf{S}\mathbf{x}^\ast$, a learned predictor of the target
null coefficients; $\gamma>0$, the weight of the null-space matching term;
$\gamma_{\mathrm{g}}>0$, the weight of the null-only graph regularizer; and
$\mathbf{P}_n$ and $\mathbf{L}$, the null projector and graph Laplacian, respectively.
The GSNR prior is
\[
\gamma\bigl\|\mathrm{G}^\ast(\mathbf{y})-\mathbf{S}\hat{\mathbf{x}}_0\bigr\|_2^2
\;+\;
\gamma_{\mathrm{g}}\bigl\|\mathbf{P}_n \mathbf{L}\hat{\mathbf{x}}_0\bigr\|_2^2,
\]
for DPS (acting on the current prediction $\hat{\mathbf{x}}_0$), and the analogous expression
with $\mathbf{x}$ in the DiffPIR proximal subproblem.

\begin{algorithm}[!t]
\footnotesize
\caption{GSNR–MPGD Sampling \textcolor{OliveGreen!100}{null-space} and \textcolor{blue!100}{graph regularization}}
\label{alg:gsnr_mpgd}
\begin{algorithmic}[1]
\Require $K,\;\mathbf{H},\;\mathbf{y},\;\{\,\zeta_i\,\}_{i=1}^K,\;\{\,\tilde{\sigma}_i\,\}_{i=1}^K,\;
         \gamma,\;\gamma_{\mathrm{g}},\;\mathbf{S},\;\mathrm{G}^\ast,\;\mathbf{P}_n,\;\mathbf{L}$
\State $\mathbf{z}_K \sim \mathcal{N}(\mathbf{0},\mathbf{I})$
\For{$i = K-1,\dots,0$}
\State $\epsilon_i \sim \mathcal{N}(0,\mathbf{I})$
\State $\mathbf{z}_{0|i} = \frac{1}{\sqrt{\overline{\alpha}_i}} (\mathbf{z}_i - \sqrt{1-\overline{\alpha}_i}\epsilon_\theta(\mathbf{z}_i,i)$

\State $\mathbf{z}_{0|i} = \mathbf{z}_{0|i} -  \zeta_i\left( \nabla_{ \mathbf{z}_{0|i}}\Vert \mathbf{H}\mathrm{D}(\mathbf{z}_{0|i})-\mathbf{y}\Vert +           \right.
$

\State $ \left. \qquad \;\; \textcolor{OliveGreen!100}{\gamma\,\bigl\|\mathrm{G}^\ast(\mathbf{y})-\mathbf{S}\mathrm{D}(\mathbf{z}_{0|i})\bigr\|_2^2} + \textcolor{blue!100}{\gamma_{\mathrm{g}}\,\bigl\|\mathbf{P}_n \mathbf{L}\mathrm{D}(\mathbf{z}_{0|i})\bigr\|_2^2}\right)$
\State $\mathbf{z}_{i-1} = \sqrt{\overline{\alpha}_{i-1}} \mathbf{z}_{0|i} + \sqrt{1-\overline{\alpha}_{i-1}-\sigma_i^2}\epsilon_\theta(\mathbf{z}_i,i)+\sigma_i\epsilon_i$
\EndFor
\State \Return $\hat{\mathbf{x}} = \mathrm{D}(\mathbf{z}_0)$
\end{algorithmic}
\end{algorithm}

\medskip
\subsection{MPGD \cite{he2024manifold}.}
We evaluate graph-null-space-regularized manifold projected gradient descent (MPGD) \cite{he2024manifold} for super-resolution. This DM is performed on the latent space \cite{rombach2022high}. We used the pre-trained latent diffusion model from \footnote{\href{http://github.com/CompVis/latent-diffusion}{github.com/CompVis/latent-diffusion}} with the CelebA-HQ model. In Algorithm \ref{alg:gsnr_mpgd}, we show the GSNR modification of MPGD. For the experiments, we used the $\mathbf{L}=\mathbf{L}_{8nn}$ variant with $p=0.1 n$. \textcolor{black}{In Fig.~\ref{fig:dm_mpgd}, we present the visual outcomes of incorporating GSNR into MPGD. This integration yields up to 0.78 dB SR improvement, indicating that GSNR enhances even competitive end-to-end diffusion-based solvers.}

\begin{figure}[!t]
    \centering
    \includegraphics[width=\linewidth]{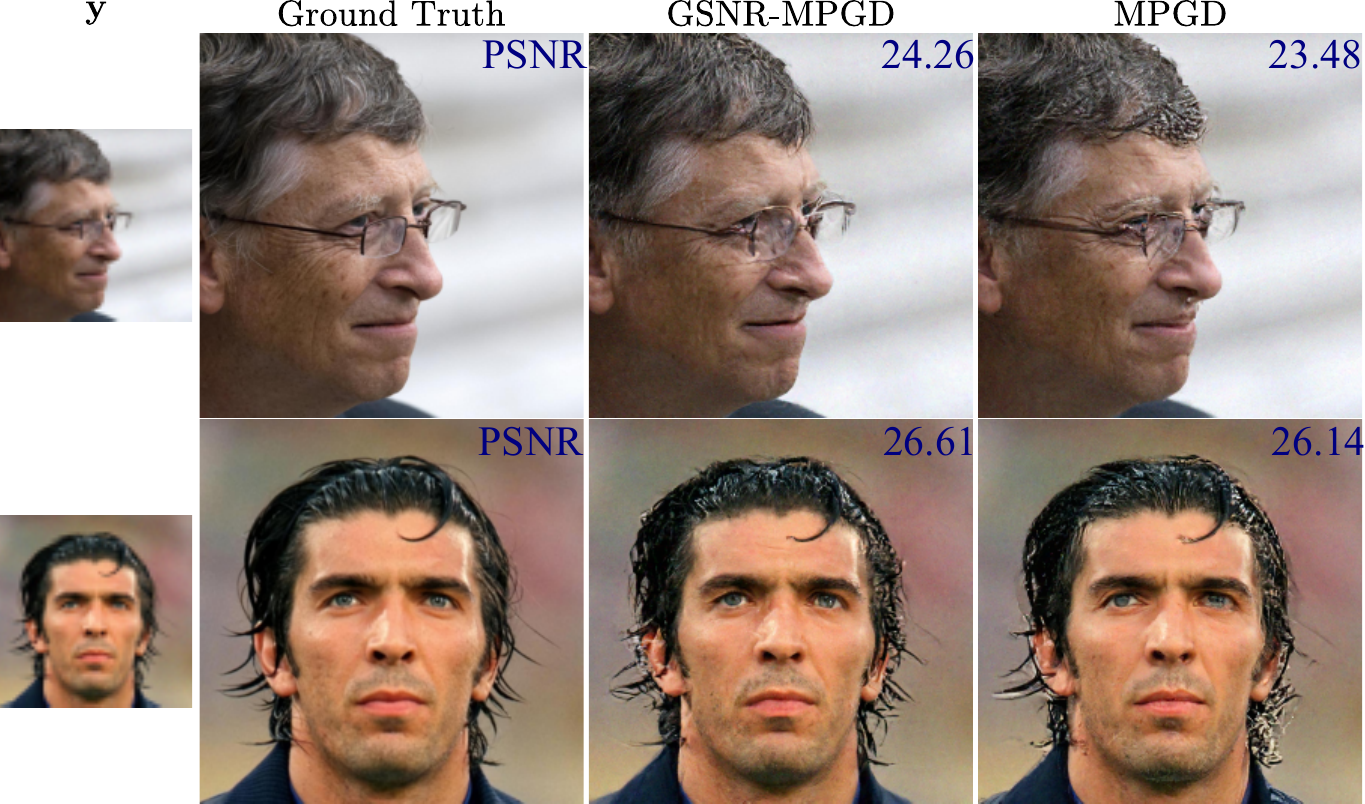}
    \caption{Results using latent-space diffusion models for SR.}
    \label{fig:dm_mpgd}
\end{figure}

\section{CS results}
\label{app:cs_results}

\textcolor{black}{Fig. \ref{fig:cost_function_ablation} shows an ablation of \eqref{eq:recon} for the convergence of the reconstruction in PnP-FISTA for CS. PnP Baseline indicates that $\gamma=\gamma_g=0$. For the GSNR, with only the graph-regularizer (green), $\gamma=0$ is used. The state-of-the-art, NPN baseline (red) uses the matrix $\mathbf{S}$ from \cite{Neurips} and $\gamma_g=0$, obtains greater acceleration. GSNR without graph-regularization (blue) uses the proposed operator $\mathbf{S}$ from \eqref{eq:S_definition} and the learned network $\mathrm{G}$ with $\gamma_g=0$. Finally, GSNR with both regularizers, in yellow, utilizes the entire equation \eqref{eq:recon}. From this figure, we can conclude the contribution of each term: if $\tfrac{\gamma_g}{2}\,(\Pin \mathbf{x})^\top \mathbf{L}(\Pin \mathbf{x})$ is used, better convergence and reconstruction will be achieved. If $\gamma\,\|\mathrm{G}(\mathbf{y})-\mathbf{S} \mathbf{x}\|_2^2$ is used, convergence is further accelerated and reconstruction is improved. Using both terms guarantees good convergence and the best possible reconstruction quality (yellow).}

\begin{figure}[!t]
    \centering
\includegraphics[width=0.8\linewidth]{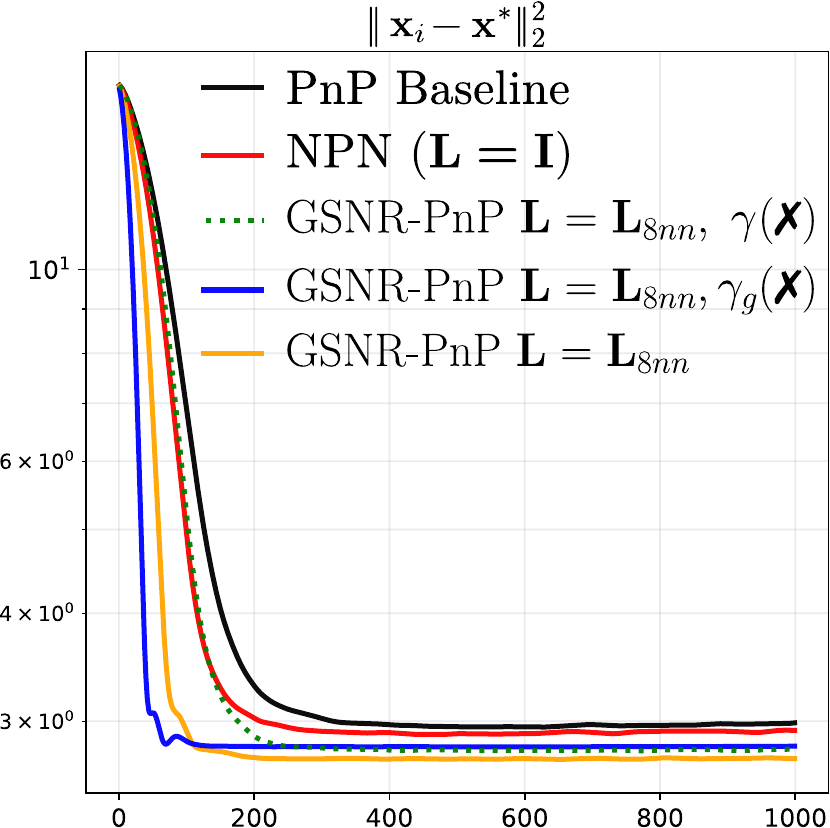}
\vspace{-0.2cm}
\caption{Convergence comparison with PnP, NPN, and GSNR for CS.}
\vspace{-0.3cm}
\label{fig:cost_function_ablation}
\end{figure}

\section{Demosaicking results}
\label{app:demosaicking}

\begin{table}[!t]
\caption{Experimental settings  for GSNR PGD-PnP in demosacing.} 
\vspace{-0.3cm}
\label{tab:parameters_demo}
\centering
\begin{tabular}{lc}
\toprule
Parameter & Value\\
\midrule
$\gamma_{\mathrm{g}}$&0.01\\
$\gamma$ &0.1\\
$\eta$&0.01\\
$\alpha$&$\frac{1.0}{\|\mathbf{H}\|}$\\
$K$ (Max iterations)&3000\\
\bottomrule
\end{tabular}%
\end{table}
In Table \ref{tab:parameters_demo}, the parameter settings are shown for demosaicing experiments.
As shown in Table \ref{tab:psnr_compact}, GSNR consistently improves over the baselines. With Lip-DnCNN, GSNR with either $\mathbf{L}_{4\text{nn}}$ or $\mathbf{L}_{8\text{nn}}$ yields a clear PSNR gain over both PGD-PnP and NPN, showing that graph-limited null-space information brings a systematic boost. With DRUNet, all methods achieve similar absolute PSNR, but GSNR still matches or slightly surpasses NPN across graph topologies. Overall, the improvements are modest yet consistent, underscoring that the main advantage comes from the GSNR null-space representation rather than the specific choice of backbone denoiser or additional graph penalty.
\begin{table}[!t]
\caption{Final PSNR (dB) comparison across graph variants and topologies for demosaicing with $n=64^2$ and $p=0.5\cdot 3\cdot n$.}
\vspace{-0.3cm}
\label{tab:psnr_compact}
\centering
\resizebox{\linewidth}{!}{\begin{tabular}{lrcc}
\toprule
Method & $\gamma_{\mathrm{g}}$ & Lip-DnCNN \cite{ryu2019plug} & DRUNet  \cite{zhang2021plug} \\
\midrule 
Baseline & -- & 39.35 & 27.91\\
NPN \cite{Neurips}        & \ding{56} & 39.77 & 30.12 \\
GSNR w. $\mathbf{L}_{4nn}$  & \ding{56} & 39.79 & 30.14 \\
GSNR w. $\mathbf{L}_{4nn}$  & \ding{51} & 39.89 & 30.13 \\
GSNR w. $\mathbf{L}_{8nn}$  & \ding{56} & 39.77 & 30.27 \\
GSNR w. $\mathbf{L}_{8nn}$  & \ding{51} & 39.88 & 30.27 \\
\bottomrule
\end{tabular}}
\end{table}

\section{SR results} \label{app:sr_results}
\textbf{SAR details for SR:} from which we extract 3289 patches of size $128\times128$ from six satellite scenes for training and 1162 patches from two additional satellite scenes for evaluation. 

\begin{figure}[!t]
    \centering
    \includegraphics[width=\linewidth]{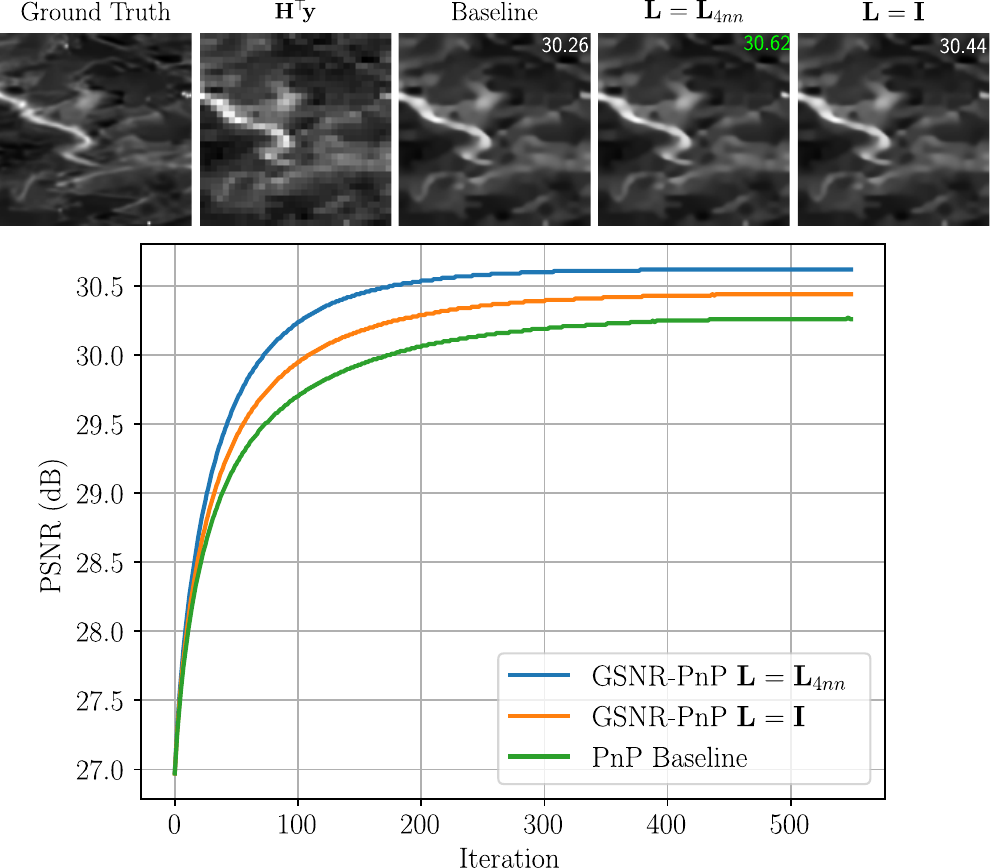}
    \vspace{-0.6cm}
    \caption{Evaluation of GSNR-PnP on a non-optical imaging example (SAR-like data).
    Top row: ground-truth image $\mathbf{x}^\ast$, back-projection $\mathbf{H}^\top\mathbf{y}$,
    baseline PnP reconstruction, and GSNR-PnP with $\mathbf{L}_{4\mathrm{nn}}$ and
    $\mathbf{L}=\mathbf{I}$, with PSNR values overlaid. 
    Bottom: PSNR as a function of iteration for GSNR-PnP with $\mathbf{L}_{4\mathrm{nn}}$,
    GSNR-PnP with $\mathbf{L}=\mathbf{I}$, and the PnP baseline.}
    \label{fig:gsnr_sar}
\end{figure}

\begin{figure*}[!t]
    \centering
    \includegraphics[width=\linewidth]{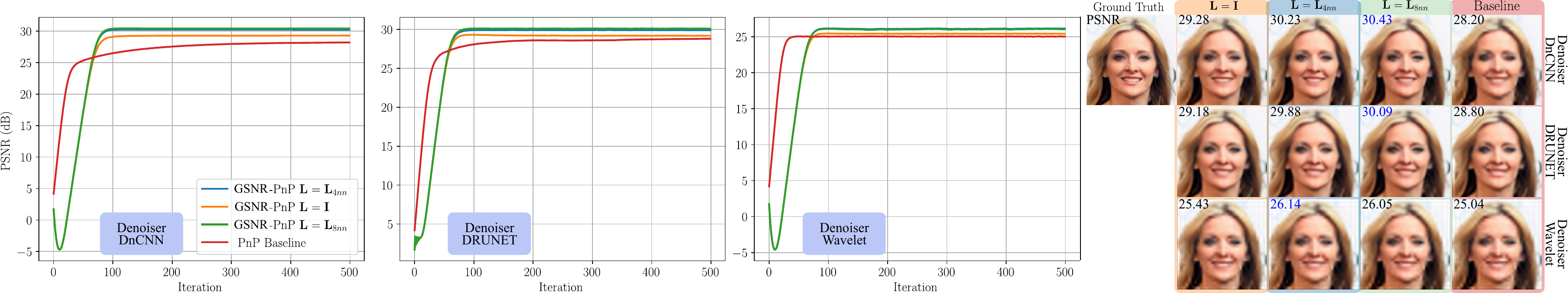}
    \vspace{-0.7cm}
    \caption{{Super-resolution results with GSNR-PnP for different graph Laplacians and denoisers.
    Left: PSNR versus iteration for GSNR-PnP with $\mathbf{L}=\mathbf{I}$, $\mathbf{L}_{4\mathrm{nn}}$,
    $\mathbf{L}_{8\mathrm{nn}}$, and the PnP, using (from left to right) DnCNN, DRUNET, and a wavelet denoiser. Right: corresponding reconstructions for a representative face image, with the ground truth at the top-left and PSNR values overlaid on each result.} }
    \vspace{-0.3cm}
    \label{fig:sr_gsnr_dynamics}
\end{figure*}

Figure~\ref{fig:gsnr_sar} demonstrates that the proposed graph-smooth null-space representation
extends beyond natural images to other imaging domains. The top-row reconstructions show that
GSNR-PnP with a grid Laplacian $\mathbf{L}_{4\mathrm{nn}}$ produces sharper, less noisy
structures than both the baseline PnP solver and the geometry-free $\mathbf{L}=\mathbf{I}$
variant, even though the scene exhibits a speckled texture rather than a smooth photographic
content. The corresponding PSNR curves in the bottom panel confirm this behavior quantitatively:
GSNR-PnP with $\mathbf{L}_{4\mathrm{nn}}$ converges faster and stabilizes at the highest PSNR,
while GSNR-PnP with $\mathbf{L}=\mathbf{I}$ still improves over the baseline but remains clearly
below the graph-based design. These results indicate that encoding graph-smooth structure in the
null-space is beneficial not only for face or natural-image SR, but also for more
challenging sensing models such as SAR-like imaging, supporting the broader applicability of
GSNR across imaging modalities.

Figure~\ref{fig:sr_gsnr_dynamics} illustrates the effect of the graph-smooth null-space design on
super-resolution performance and convergence. The PSNR–iteration plots (left) show that, for all
three denoisers, GSNR-PnP with either $\mathbf{L}_{4\mathrm{nn}}$ or $\mathbf{L}_{8\mathrm{nn}}$
converges to a higher PSNR than both the baseline PnP solver and the geometry-free
$\mathbf{L}=\mathbf{I}$ variant. With DnCNN and DRUNET, the graph-based curves not only reach a
higher plateau but also exhibit a steeper initial rise, indicating a faster approach to a good
solution. Even with the simpler wavelet denoiser, the graph-limited versions match or slightly
improve upon the baseline while maintaining stable dynamics. The image grid on the right confirms
these trends visually: reconstructions obtained with $\mathbf{L}_{4\mathrm{nn}}$ and
$\mathbf{L}_{8\mathrm{nn}}$ display sharper facial contours and more coherent high-frequency
details than both the baseline and $\mathbf{L}=\mathbf{I}$, highlighting the benefit of injecting
graph-smooth structure specifically into the null space.

\begin{figure}[ht]
    \centering
    \includegraphics[width=\linewidth]{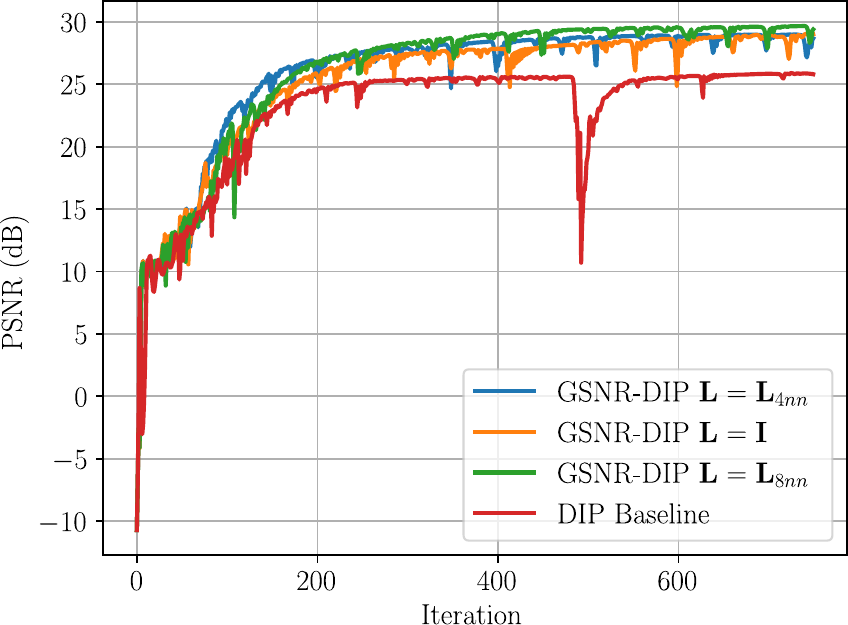}
    \vspace{-0.6cm}
    \caption{Super-resolution on CelebA (20 images) with Deep Image Prior (DIP). We plot PSNR versus iteration for the standard DIP baseline (red) and the proposed GSNR-DIP variants using different Laplacians: $\mathbf{L}=\mathbf{I}$,
    $\mathbf{L}_{4\mathrm{nn}}$, and $\mathbf{L}_{8\mathrm{nn}}$.}
    \vspace{-0.2cm}
    \label{fig:gsnr_dip_sr}
\end{figure}

Figure~\ref{fig:gsnr_dip_sr} shows the evolution of PSNR for super-resolution on 20 images from the CelebA dataset when DIP is run with and without the proposed graph-smooth null-space representation.
The DIP baseline (red) improves rapidly at first but then saturates at a lower PSNR and
exhibits noticeable instabilities, including pronounced dips during the late iterations.
In contrast, all GSNR-DIP variants converge to a higher PSNR plateau and have much smoother
trajectories. Among them, the graph-based choices $\mathbf{L}_{4\mathrm{nn}}$ and
$\mathbf{L}_{8\mathrm{nn}}$ (blue/green) provides the most stable and accurate reconstructions,
consistently outperforming both the baseline and the geometry-free $\mathbf{L}=\mathbf{I}$
(orange). This indicates that enforcing graph-smooth structure specifically in the null space
not only improves the final reconstruction quality but also regularizes the DIP optimization
itself, mitigating the overfitting and oscillations typically observed in vanilla DIP.

\section{Deblurring results}\label{app:deblurring}
In the deblurring setting, the sensing matrix $\mathbf{H}$ does not reduce dimensionality, meaning that $n-m=0$. Consequently, the full set of $n$ eigenvectors of $\mathbf{T}$ was considered when selecting the $p$ smoothest directions.

\begin{table}[!t]
\caption{Experimental settings for GSNR PnP-PGD in deblurring.}   
\vspace{-0.3cm}
\label{tab:parameters_deblur}
\centering
\begin{tabular}{lccc}
\toprule
Parameter & Lip-DnCNN &Wavelet\\
\midrule
$\gamma_{\mathrm{g}}$&0.1&0.1\\
$\gamma$ &0.1&0.1\\
$\eta$&0.0001&0.001\\
$\alpha$&$\frac{1.5}{\|\mathbf{H}\|}$&$\frac{1.5}{\|\mathbf{H}\|}$\\
$K$ (Max iterations)&800&800\\

\bottomrule
\end{tabular}%
\vspace{0.1cm}
\end{table}

\textbf{Experimental settings.}
Table \ref{tab:parameters_deblur} reports the GSNR PnP-PGD parameters defined to perform image reconstruction.

\begin{figure*}[!t]
    \centering
    \includegraphics[width=\linewidth]{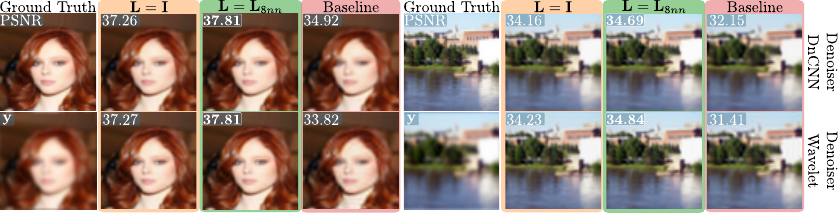}
    \vspace{-0.7cm}
    \caption{Deblurring visual results comparing PGD-PnP, GSNR with $\mathbf{L}=\mathbf{I}$ (NPN), and GSNR with $\mathbf{L}=\mathbf{L}_{8nn}$ across two datasets and two denoisers. Best results for each denoiser are in \textbf{bold}.
    Here, $p=0.8$.}
    \vspace{-0.5cm}
    \label{fig:visual_deblur}
\end{figure*}

 \begin{table}[!t]
  \caption{\textcolor{black}{GSNR's computational cost for SR: offline and online (PnP, one image, 1000 iters).}}
  \vspace{-0.3cm}
  \centering
  \resizebox{\linewidth}{!}{
 \begin{tabular}{lccc}
\toprule
\textbf{Metric \quad \quad / \quad \quad Resolution} & $128^2$ & $256^2$ & $512^2$\\
\midrule
Offline EVD computation (s) & 102 & 1315 & 22310 \\
Offline RAM (GiB) & 0.68 & 10.88 & 174.09  \\
PnP runtime (Wavelet) (s) & 1.09 & 1.42 & 1.96 \\
\bottomrule
\end{tabular}
}
  \label{tab:EVD}
    \vspace{-0.7cm}
\end{table}

\textbf{Visual results.}
Fig.~\ref{fig:visual_deblur} presents deblurring examples on the CelebA and Places365 datasets for baseline PGD-PnP, GSNR with $\mathbf{L}=\mathbf{I}$ (NPN), and GSNR with $\mathbf{L}=\mathbf{L}_{8nn}$, using either Lip-DnCNN or Wavelet denoisers. The qualitative behavior is consistent across both denoisers. The baseline PGD-PnP produces reconstructions with noticeable residual blur and poor recovery of high-frequency structures, such as eyes and eyebrows in faces or the sharp contours of tree tops in natural scenes.


Using GSNR with $\mathbf{L}=\mathbf{I}$ improves the reconstruction of high-frequency components, producing visually sharper results. However, because the identity Laplacian imposes no geometric constraints, the recovered details are not necessarily aligned with the true image structures. This often leads to overly sharp but inaccurate features that do not reflect the desired image geometry. In contrast, GSNR with graph-based Laplacians, such as $\mathbf{L}_{8nn}$, enforces geometric consistency through graph-smoothness. As a result, the restored high-frequency structures are both sharp and coherent with the true image content, yielding the most faithful reconstructions among the tested methods. This advantage is reflected not only visually but also quantitatively, with GSNR achieving the highest PSNR values. A noteworthy observation is that GSNR enables reconstructions using a simple Wavelet denoiser to surpass the performance of reconstructions obtained with a more advanced neural-network denoiser like DnCNN. This highlights GSNR’s capacity to adapt to and significantly strengthen any reconstruction method, offering robust improvements across a broad range of inverse problem solvers.

\begin{table}[!t]
\caption{Final PSNR on DIV2K for deblurring ($\mathrm{G}$ trained on Places365).}
\vspace{-0.3cm}
\centering
\resizebox{0.8\linewidth}{!}{
 \begin{tabular}{lcc}
\toprule
& Lip-DnCNN & Wavelet \\
\midrule
Baseline (PnP-PGD) &31.23 & 30.39\\
NPN \cite{Neurips}   & 33.05 & 32.90 \\
GSNR \ ($\mathbf{L}_{8\text{nn}}$) & \best{33.69} & \best{33.65}  \\

\bottomrule
\end{tabular}
}
\label{tab:cross_dataset_div2k_DB}
  \vspace{-0.3cm}
\end{table}

\section{Scalability and computational cost}\label{app:scalability}
We use $128^2$ images to enable controlled comparisons. GSNR introduces only a low \textit{online} overhead at inference time, since the subspace computation is performed \textit{offline} once per $(\mathbf{H, L}, n, p)$ and reused. 
We only compute the $p$ smoothest eigenvectors (selected by the coverage plateau in Algorithm \ref{alg:coverage_p}), so memory scales as $\mathcal{O}(pn)$. We never materialize $\mathbf{P}_n$ or $\mathbf{T}$: we wrap $\mathbf{Tx}=\mathbf{P}_n\mathbf{L}\mathbf{P}_n \mathbf{x}$ as an implicit LinearOperator and compute the first $p$ eigenpairs via ARPACK \texttt{eigsh} (\ref{app:t_constructio}), exploiting sparse $\mathbf{L}$ for fast matrix operations.   Table ~\ref{tab:EVD} reports scaling to $512^2$; the online cost remains low, whereas the offline EVD increases with resolution. For $n^2\geq 1024^2$, EVD computation becomes costly with ARPACK. Future work will focus on efficient approximate EVD computation.

\section{Dataset generalization and neural network ablation}
\label{app:neural_network_ablation}
\textcolor{black}{To assess whether GSNR regularization transfers across image distributions, we performed a cross-dataset experiment for deblurring setting. We trained the neural network $\mathrm{G}$ on Places365 and then evaluated GSNR on DIV2K \cite{DIV2K}. The results in Table~\ref{tab:cross_dataset_div2k_DB} show that the gains provided by GSNR remain consistent under this distribution shift, indicating that the learned null-space predictor generalizes to diverse natural images.}

\textcolor{black}{We further examine the robustness of GSNR with respect to the choice of neural network architecture for $\mathrm{G}$. In Figure \ref{fig:G_ablation}, we replace the default U-Net with three recent architectures: DFPIR \cite{tian2025degradation}, EVSSM \cite{EVSSM}, and AdaIR \cite{cui2025adair}. As expected, stronger predictors yield more accurate null-space estimates, as reflected by the higher PSNR values for the null-space mapping (see Fig.~\ref{fig:G_ablation} legend). Importantly, GSNR improves PnP reconstruction quality and accelerates convergence for all tested backbones, suggesting its benefits arise from the proposed GSNR regularization rather than a specific architecture.}

\begin{figure}[t]
    \centering
\includegraphics[width=\linewidth]{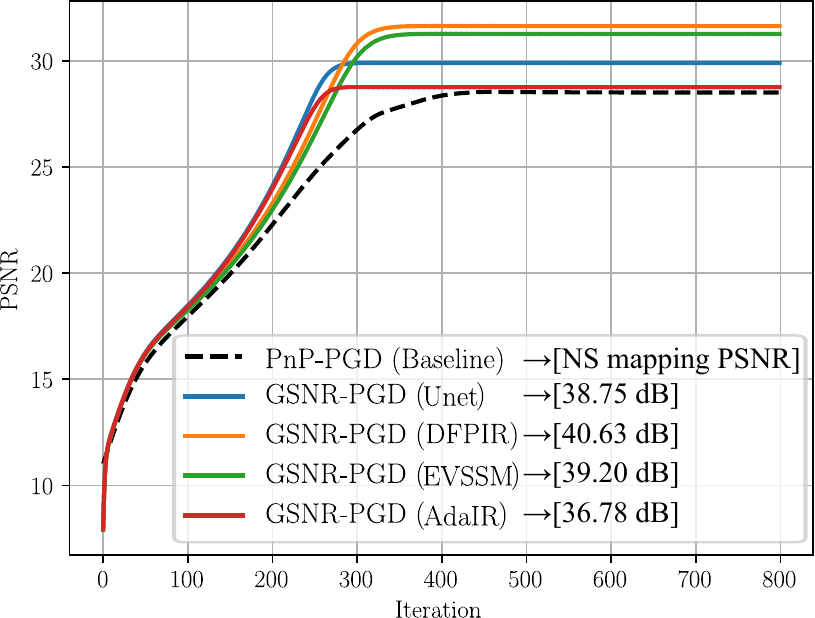}
\caption{PSNR with $\mathrm{G}$ ablation in Places for the deblurring task.}
\label{fig:G_ablation}
\end{figure}



\section{Inexact forward operator}\label{app:inexact_forward_operator}
\vspace{-0.2cm}
In many inverse problems, the forward operator available to the reconstruction algorithm is only an approximation of the true sensing physics. In practice, the nominal sensing matrix $\mathbf{H}$ can deviate from the actual measurement operator due to calibration errors, hardware tolerances, or other unmodeled effects. Such mismatches are particularly detrimental because they propagate into the algorithmic components derived from $\mathbf{H}$, including null-space projectors and any structure imposed through them, and can therefore undermine both accuracy and convergence. To quantify this effect in a concrete setting, Fig.~\ref{fig:imperfect_H} evaluates deblurring under operator mismatch. Measurements are generated as \(\mathbf{y=(H+H_{\xi})x^{\ast}+\boldsymbol{\omega}}\), with \(\mathbf{H_{\xi}}\sim\mathcal{N}(0,0.005^2)\) while recovery still assumes the nominal $\mathbf{H}$. This mismatch yields imperfect estimates of \(\mathbf{P}_{n}\) and \(\mathbf{S}\), which in turn degrades the null-space representation used by GSNR. Despite these compounded imperfections, GSNR remains effective, improving performance by approximately 1 dB and converging in fewer iterations.


\begin{figure}[!t]
\centering
\includegraphics[width=\linewidth]{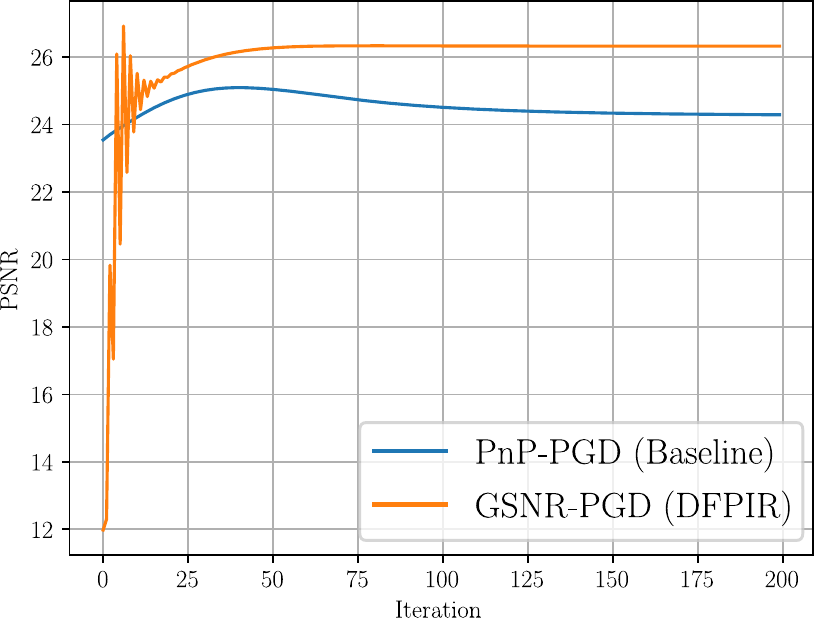}
\captionof{figure}{PSNR in deblurring with inexact forward operator.}
\label{fig:imperfect_H}
\end{figure}





